\documentclass{article}

\usepackage{times}
\usepackage{graphicx} % more modern
\usepackage{subfigure} 

\usepackage{natbib}

\usepackage{algorithm}
\usepackage{algorithmic}

\usepackage{hyperref}

\usepackage{amsmath,amsbsy,amsfonts,amssymb,amsthm,units}
\usepackage{graphicx}
\usepackage{color,cases}
\usepackage{tikz}
\usetikzlibrary{calc,shapes}
\usepackage{subfigure}
\usepackage{bbm}
\usepackage{xcolor}
\definecolor{DarkGreen}{rgb}{0.1,0.5,0.1}
\definecolor{DarkRed}{rgb}{0.5,0.1,0.1}
\definecolor{DarkBlue}{rgb}{0.1,0.1,0.5}
\usepackage{hyperref}
\usepackage{nicefrac}
\usepackage{algorithm}
\usepackage{algorithmic}
\usepackage[accepted]{icml2016}

\icmltitlerunning{Provable algorithms for inference in topic models}
\newcommand{\barkappa}{\bar{\kappa}}

\newcommand{\ltau}{\tau}
\newcommand{\cat}{\textup{cat}}
\newcommand{\inftybound}{\lambda_{\delta}(A)}

\newtheorem{theorem}{Theorem}[section]

\newtheorem{lemma}[theorem]{Lemma}
\newtheorem{claim}[theorem]{Claim}
\newtheorem{corollary}[theorem]{Corollary}
\newtheorem{assumption}[theorem]{Assumption}
\newtheorem{proposition}[theorem]{Proposition}
\theoremstyle{definition}

\newtheorem{definition}[theorem]{Definition}

\newtheorem{remark}{Remark}
\newcommand\E{\mathbb{E}}

\newcommand\inner[1]{\langle #1 \rangle}

\newcommand\supp{\operatorname{supp}}

\newcommand{\Exp}{\mathop{\mathbb E}\displaylimits}

\def\shownotes{0}  %set 1 to show author notes
\ifnum\shownotes=1
\newcommand{\authnote}[2]{{\marginpar{{\color{DarkRed} #1}}$\ll$\textsf{\footnotesize #1 notes: #2}$\gg$}}
\else
\newcommand{\authnote}[2]{}
\fi

\newcommand{\vecinftynorm}[1]{|#1|_{\infty}}
\newcommand{\veconenorm}[1]{|#1|_{1}}

\newcommand{\diffmle}{u}
\newcommand{\hessian}{\nabla^2}
\newcommand{\mle}{{\small \textup{MLE}}}
\newcommand{\one}{\mathbf{1}}
\newcommand{\mper}{\,.}
\newcommand{\xstar}{x^*}

\newcommand{\hata}{\hat{a}}

\newcommand\R{\mathbb{R}}

\newcommand{\Id}{\textup{Id}}

\newcommand{\trace}{\textup{tr}}

\newcommand{\simplex}{\mathcal{S}}

\begin{document} 

\twocolumn[
\icmltitle{Provable Algorithms for Inference in Topic Models}

\icmlauthor{Sanjeev Arora}{arora@cs.princeton.edu} %\footnote{.wewer}
\icmladdress{Department of Computer Science, Princeton University}
\icmlauthor{Rong Ge}{rongge@cs.duke.edu}
\icmladdress{Computer Science Department, Duke Unversity}
\icmlauthor{Frederic Koehler}{fkoehler@princeton.edu}
\icmladdress{Department of Mathematics, Princeton University}
\icmlauthor{Tengyu Ma}{tengyu@cs.princeton.edu}
\icmladdress{Department of Computer Science, Princeton University}
\icmlauthor{Ankur Moitra}{moitra@mit.edu}
\icmladdress{Department of Mathematics and CSAIL, Massachusetts Institute of Technology }

\icmlkeywords{topic models, provable inference algorithms}

\vskip 0.3in]

\begin{abstract} 
Recently, there has been considerable progress on designing algorithms with provable guarantees --- typically using linear algebraic methods --- for parameter learning in latent variable models. 
But designing provable algorithms for inference has proven to be more challenging. Here we take a first step towards provable inference in topic models. We leverage a property of topic models that enables us to construct simple linear estimators for the unknown topic proportions that have small variance, and consequently can work with short documents. Our estimators also correspond to finding an estimate around which the posterior is well-concentrated. We show lower bounds that for shorter documents it can be information theoretically impossible to find the hidden topics. Finally, we give empirical results that demonstrate that our algorithm works on realistic topic models. It yields good solutions on synthetic data and runs in time comparable to a {\em single} iteration of Gibbs sampling. 
\end{abstract}

\section{Introduction}

%\subsection{Background}

Generative models of data are ubiquitous in unsupervised learning, and lead to two types of computational problems: In {\em parameter learning}, the goal is find the parameters of the model that best fits a given collection of data. In {\em inference}, the goal is to learn the values of latent variables for a specific datapoint. A wide range of approaches are empirically effective for both tasks, including {\em Gibbs sampling} and {\em variational inference}. However, for the most part we lack %any sort of 
strong provable guarantees --- on running time, or quality of solution --- for these approaches.
% \authnote{RG}{I wouldn't say ``any sort'' because they don't have some useless guarantees} 
%They seem to work, but we do not yet understand how to analyze them or when and why they work. 

Recently, there has been considerable progress on designing new algorithms for parameter learning
with such provable guarantees.
Since the usual maximum likelihood estimator is often NP-hard to compute even in simple models, these new algorithms use alternative estimators based on the method of moments and linear algebra. Their analysis usually involves making a structural assumption about the parameters of the problem, 
which can often be justified in applications. Some highlights include algorithms %with provable guarantees --- i.e. bounds on their running time and sample complexity --- 
for {\em topic modeling}~\citep{AGM,AFHKL}, {\em learning mixture models}~\citep{MV:many_gaussians, HK12,GeHK15}, community detection~\citep{AGHK} and (special cases of) {\em deep learning}~\citep{ABGM14,JSA15}.

But there has been comparatively much less progress on designing algorithms with provable guarantees for inference. The current paper takes a first step in this direction, in context of topic models. Our algorithms leverage a property of topic models (Definition~\ref{def:delta}) that turns out to hold in many datasets --- the existence of a good approximate inverse matrix. 
%We introduce it later in the introduction. However, it is a natural adaptation of ideas from regression to the setting where vectors have unit $\ell_1$-norm and there is large sampling error. Under this condition, we are able to construct simple linear estimators with low variance that gives reasonable estimation even with small number of samples.% provably solve the inference problem.
We also give empirical results that demonstrate that our algorithm works on realistic topic models. On synthetic data, its error is competitive with state-of-the-art approaches (which have no such provable guarantees). It obtains somewhat weaker results on real data. % \authnote{FK}{Maybe too strong, performance on synthetic data and real data are more correlated than this suggests. (It makes it sound like performance on synthetic data is way better which I don't think is true, the performance on synthetic data is not that amazing.) How about something more like ``and slightly worse guarantees on real data.''?}
\subsection{Setup and Overview}\label{subsec:setup}

Here we describe topic modeling, and why inference appears more difficult than 
%difference between 
parameter learning. In topic modeling, each document is represented as a {\em bag of words} where we ignore the order in which words occur. The model assumes there is a fixed set of $k$ {\em topics}, each of which is a distribution on words. Thus the $i$th topic is a vector $A_i \in \R^D$ (where $D$ is the number of words in the language) whose coordinates are nonnegative and sum to $1$. Each document is generated by first picking its topic proportions from some distribution; say $x_i$ is the proportion of topic $i$, so that $\sum_i x_i=1$. The model assumes a distribution on $x$ that favors sparse or approximately sparse vectors; a popular choice is the Dirichlet distribution~\citep{LDA}. Then the document $\{w_1, w_2, \ldots, w_n\}$ is generated by drawing $n$ words  independently  from the distribution $A \cdot x$ where $A$ is the matrix whose columns are the topics. It is important to note that the document size $n$ can be quite small (e.g., $n$ may be $400$, and  $D$ may be $50,000$) so the empirical distribution of words in a document is in general a very {\em inaccurate} approximation to $Ax$. With some abuse of notation we also think of $y$ as a vector in $\R^D$, whose $j$th coordinate is the number of occurences of word $j$ in the document.

Parameter learning involves recovering the best $A$ for a corpus of documents; this can be seen as a the latent structure in the corpus. Recent (provable) algorithms for this problem~\citep{AFHKL, AGM} use the {\em method of moments},  leveraging the fact that some form of averaging over the corpus yields a linear algebraic problem for recovering $A$. For example the {\em word-word} co-occurence matrix (whose $i, j$ entry is the probability that words $i, j$ co-occur in a document) is given by $$E_x[Axx^TA^T] = AZA^T$$ where $Z$ is the 2nd moment matrix of the prior distribution on $x$. It is possible to recover $A$ from this expression, under natural conditions like {\em separability}~\cite{AGM}. Alternatively, one can use a co-occurrence tensor and recover $A$ under weaker assumptions~\cite{AFHKL}. 

In the inference problem, we know the topic matrix $A$ and are given a single document $y$
generated using this matrix. The goal is to find the posterior distribution $x|y$. This can be seen as labeling or categorizing this document, which is important in applications.  Inference is reminiscent of classical regression problems where the goal is to find $x$  given $y = A x + \mbox{noise vector}$. The key difference here is the nature of noise ---for each word coordinate $j$ is $1$ with probability $(Ax)_j$, and $0$ otherwise--- which means that the noise on a coordinate-by-coordinate basis can be much larger than the signal. In particular the vector $y \in \R^D$ is very sparse even though $Ax$ is dense.
%\Tnote{above sentences above seems to be disconnected with other sentences before}
\iffalse 
The basic issue is that in inference we have to estimate $x$ for a particular document, whereas in learning we do not because we can average over many documents. This problem is compounded by the fact that
\fi  
This problem can be seen as an analog of sparse linear regression  when the target (regression) vector $x$ has nonnegative coordinate and $\sum_i x_i =1$. 
%\footnote{So this is quite different from  the 
(This is distinct from usual $\ell_1$-regression where regression vector is in $\ell_2$ even though the loss function is $\ell_1$.) The difficulty here, in addition to the issue of high coordinate-wise error already mentioned, is that  the usual sparsity-enforcing $\ell_1$-regularization buys nothing since the solution needs to exactly satisfy $\|x\|_1 =1$.
% always has $\ell_1$ norm $1$.
% c) especially for analyzing the maximum likelihood estimator (MLE), to get tight bound, one has to use fisher information matrix to convert between $\ell_2$ to $\ell_1$ carefully. 

Inference seems more difficult than parameter learning because
averaging over many documents is no longer an option. Furthermore, the solution $x$ is not unique in general,  and in some cases the posterior distribution on $x$ is not well concentrated around any particular value. (In practice Gibbs Sampling can be used to sample from the posterior~\citep{griffiths2004finding,LDAinference}, but as mentioned,  a rigorous analysis has proved difficult. The inference is actually NP-hard.)
We will view inference as a problem of recovering some
{\em ground truth} $x^*$ that was used to generate the document, and we show that with probability close to $1$ our estimate $\hat{x}$ is close to $x^*$ in $\ell_1$ norm. 

\paragraph{Bayesian vs Frequentist Views.} 
So far we have not differentiated between Bayesian and frequentist approaches to frame the inference problem, and now we show that the two are closely related here. The above description is frequentist, assuming an unknown \textquotedblleft ground truth\textquotedblright vector $x^*$ of topic proportions (which is $r$-sparse for some small $r$) was used to generate a document $y$, using a distribution $y|x^*$. Let $\mathcal{E}_{x^*}$  be the event that
our algorithm recovers a vector $\hat{x}$ such that $\|\hat{x} -x^*\|_1 \leq \epsilon$. 
For our algorithm $\Pr_{y|x^*}[\mathcal{E}_{x^*}] \geq 1-\delta^2$ for some $\delta>0$.
By contrast, in the Bayesian view, one assumes a prior distribution on $x^*$ and seeks to output a sample from the conditional distribution $x^*|y$. 
Now we show that the success of our frequentist algorithm implies that the posterior $x^*|y$ must also be concentrated, and place most probability mass
on set of $x$ such that $\|x -\hat{x}\|_1 \leq \epsilon$.   
%Indeed, in the probability space of 
%all $x^*$, let $E$ be the event that \textquotedblleft {\em given a document $y$ generated %according to $x^*$, 
%the algorithm returns $x$ such that $\|x -x^*\|_1 \leq \epsilon$.}\textquotedblright~
%The reason is that even if $x^*$ is drawn from some distribution on $r$-sparse vectors, the %algorithm's analysis implies $\Pr_{y\vert x^*}\left[\mathcal{E}_{x^*}\right] \ge 1-\delta^2$.
By law of total expectation, we have $\Pr_{x^*,y}\left[\mathcal{E}_{x^*}\right]  = \Pr_{x^*}\left[\Pr_{y\vert x^*}\left[\mathcal{E}_{x^*}\vert x^*\right]\right] \ge 1-\delta^2$. Switching the order of expectation, we obtain  
%\begin{equation}
%	\Pr\left[E\right]  = \Pr\left[\Pr\left[E \mid x\right]\right] \ge 1-\delta^2\,.\nonumber
%	\end{equation}
%Then by the law of total expectation, we have 	
\begin{equation}
%\Pr\left[E\right]  = 
\Pr\nolimits_{y}\left[\Pr\nolimits_{x^*\vert y}\left[\mathcal{E}_{x^*} \mid y\right]\right] \ge 1-\delta^2\,.\nonumber
\end{equation}
Then it follows by Markov argument that 
\begin{equation}
\Pr\nolimits_{y}\left[\Pr\nolimits_{x^*\vert y}\left[\mathcal{E}_{x^*} \mid y\right] \ge 1-\delta \right] \ge 1-\delta\,.\nonumber
\end{equation}
Note that the inner probability is over the posterior distribution $p_{x^*\mid y}$. But the event $\mathcal{E}_{x^*}$ only depends on the output $\hat{x}$ of the algorithm given $y$. Thus the probability is at least $1-\delta$ over choice of $y$, that $1-\delta$ of the probability mass of $x^*|y$ is concentrated in the $\ell_1$  ball of radius $\epsilon$	around the algorithm's answer $\hat{x}$.

From now on the goal of our algorithm is to  {\em recover $\xstar$ given $y$}, and we  identify conditions under which the event has probability close to $1$. 
%Thus needless to say, our algorithm works only under the settings where the posterior is indeed %concentrated. 
%	We also present a lower-bound (i.e., a specific topic model) showing that when our algorithm fails to find a topic with large weight, no other algorithm (including computing the posterior) can find the support of the correct topic vector using documents of comparable length.
%	{\sc See Lemma?? }

\paragraph{Minimum Variance Estimators (with Bias).}
Having set up the problem as above, next we consider how to {\em recover} an approximation to $x^*$ 
given a document $y$ generated with topic proportions $x^*$. 
%This is reminiscent of regression problems where the goal is to find $x^*$  given $y = A x^* + \mbox{noise vector}$. The trouble is that the noise on a coordinate-by-coordinate basis is much larger than the signal. The vector $y \in \R^D$ is sparse even though $Ax^*$ is dense. Nevertheless, our approach is through a carefully chosen linear estimator. 

%But here one is instead given a vector $y \in \R^D$ which is very sparse, and whose $j$th coordinate is set (roughly speaking) according to number of heads in a binomial distribution with bias $AX^*_j$. We can think of it as a nonlinear variant of linear regression (it is known to be NP-hard in general~\cite{}). \authnote{RG}{I don't know the citation}

Since $A$ has orders of magnitude more rows than columns, it has many left inverses to choose from. If we find any matrix $B$ where $BA$ is equal to the identity matrix, then $By$ is an unbiased estimate for $x^*$. However this estimate has high variance if $B$ has large entries, necessitating  working with only very large documents. % whose length is comparable to $B$. 
Motivated by applications to collaborative filtering, \citet{KleinbergS08} introduce the notion of the {\em $\ell_1$} condition number (see Definition~\ref{ass:l_1_condition}) of $A$, which allows them to construct a left inverse $B$ with a much smaller maximum entry.
%, and they used it to construct algorithms for collaborative filtering with provable guarantees. The difficulty is there is  no known algorithm to compute the $\ell_1$ condition number.
We introduce a weaker notion of condition number called the {\em $\ell_{\infty}$-to-$\ell_1$} {\em condition number}, which leverages the observation that even if $BA$ is {\em close} to the identity matrix it still yields a good linear estimator for $x^*$. We call $B$ an approximate inverse of $A$. Moreover it has the benefit that the condition number as well as the approximate left inverse $B$ with minimum variance can also be computed in polynomial time using a linear program (Proposition~\ref{prop:condition_number})!

In our experiments, we compute the exact condition number of word-topic matrices that were found using
standard topic modeling algorithms on real-life corpora. (By contrast, we do not know the 
$\ell_1$ condition number of these matrices.) In all of the examples, we found that the condition number is at most a small constant, which allows us to compute good approximate left inverses to the topic matrix $A$ to enable us to estimate $x^*$ even with relatively short documents. 
%\vspace{-.1in}
%\Tnote{Shall we can it estimation problem or recovery problem? I feel like estimation is most consistent with the statistic terminology}
\paragraph{Main results.}
%We defer the precise statement of our results until later in the paper. 
Our main result (Theorem~\ref{thm:maininference}) shows that when the condition number is small, it is possible to estimate $x^*$ using a combination of thresholding and a left inverse $B$ of minimum variance. Our overall algorithm runs efficiently and requires time $O(nk)$ and $\widetilde{O}(r^2)$ samples to achieve $o(1)$  error in $\ell_1$ norm  and $o(1/r)$ error in $\ell_{\infty}$ norm, where $r$ is the number of topics represented in the document. Note that we do not need to assume a particular model (e.g. uniform random) for the $r$ topics, the algorithm works even when the topics may be correlated with each other.

%\Tnote{To check if we defined $r$}

As an intermediate step, we are able to recover the support of $\xstar$ when each of its non-zero coordinates is suitably bounded away from zero. We complement this result by showing that maximizing the log-likelihood function over the recovered support can further reduce the estimation error (measured in the $\ell_1$-norm) to $\widetilde{O}(\sqrt{r/n})$ (see Section~\ref{sec:mle}).  %introduce an additional convex programming approach that, 
%given a recovered approximation to 
%$x^*$, can refine it. 
%We prove some guarantees for this approach  and 
The experiments show that it indeed yields estimates for $x^*$ with smaller error (see Section~\ref{sec:experiments}). 

Finally we show that in order to recover support of $\xstar$, it is necessary to observe $\Omega(r^2)$ words, even if $A$ is perfectly conditioned and $\xstar$ is promised to have all non-zero coordinates larger than $\Omega(1/r)$ (see Lemma~\ref{lem:inftyonecondition} for a family of such perfectly conditioned but hard instances of $A$, and Lemma~\ref{lem:lower_bound} for hard instance of $\xstar$). %{\sc link to theorem??}

Thus to sum up, our overall approach involves simple linear algebraic primitives 
followed by convex programming. For a topic model with $k$ topics, the sample complexity of our algorithms depend on $\log k$ instead of $k$. This is important in practice as $k$ is often at least $100$.
The accuracy on synthetic data is good for sparse $x$, though not quite as good as Gibbs sampling.
However, %our linear algebraic methods 
if we forgo the convex programming step we can compute a reasonable estimate for $x$
from a single matrix vector multiplication plus thresholding, which is an order of magnitude faster than finding an estimate of the same quality via Gibbs sampling. %\Anote{rephrased again}
%However, it is an order of magnitude faster than Gibbs sampling 
%as implemented in MALLET, a standard package for topic modeling %\authnote{FK}{Not really true if this includes the convex programming step. How about, ``However, with our linear algebraic methods we are able to compute a reasonable estimate to the true topic vector extremely quickly, an order of magnitude faster than finding an estimate of the same quality via Gibbs sampling.''}. 
And of course, our approach comes with a performance guarantee.

\section{Notations and Setup}

%\paragraph{Topic Model}

%We first describe the generative process for documents in topic models. 
%%First we describe the sampling procedure for documents. 
%Suppose there are $D$ words in the vocabulary and $k$ topics. Let $A\in \R^{D\times k}$ be the word-topic matrix, where the columns of $A$ describe the distributions of words for different topics. Given a topic vector $x^* \in \simplex_k$ (where $\simplex_k = \{z\in \R_{\ge 0}^k : |z|_1 = 1\}$ is the $k$-dimensional probability simplex), a document $y = \{w_1,\dots, w_N\}$ is a collection of words, each of which is drawn i.i.d from the categorical distribution defined by the probability vector $Ax^*$, which is denoted by $\cat(A\xstar)$. 

In addition to the description of topic model in Section~\ref{subsec:setup}, we introduce the following notations. We use $\simplex_k= \{z\in \R_{\ge 0}^k : |z|_1 = 1\}$ to denote the $k$-dimensional probability simplex. We assume that the true topic proportion vector $\xstar\in \simplex_k$ is $r$-sparse throughout the paper. 
Sometimes we also abuse notations and use $y$ as a $D$ dimensional vector instead of a set, in this case $y_i$ is the number of times word $i$ appears in the document. We will use $a_i^{\top}$ to denote the $i$-th row of $A$. We will use $\cat(p)$ to denoted the categorical distribution defined by probability vector $p$. Euclidean norm, $\ell_1$, $\ell_{\infty}$ norm of a vector is denoted by $\|\cdot\|$, $\|\cdot\|_1$ and $\|\cdot\|_{\infty}$ respectively.

\paragraph{Condition Numbers of Matrices}

Condition number of a matrix usually represents the ratio of the largest and smallest singular values. However, this concept is tied to $\ell_2$ norm, and for probability distributions the most natural norms are $\ell_1$ and $\ell_\infty$. 

Next we define various matrix norms that we will utilize. Let $\vecinftynorm{A}= \max_{i,j} |A_{ij}|$ denotes the maximum absolute value of the entries of the matrix $A$, and $\veconenorm{A}= \sum_{i,j} |A_{ij}|$ denotes the sum of the absolute value of the entries of the matrix $A$. Let $\Id_k$ denotes the identity matrix of dimension $k$. For a matrix, let $\|\cdot \|$ denote the spectral norm,  and $\|\cdot \|_Q$ denote the norm defined by $\|x\|_Q = \sqrt{x^{\top}Qx}$ where $Q$ is a positive semidefinite matrix. We will use this norm particularly with $Q$ being fisher information matrix.  

We will also work with various notions of condition number, that we will use in our guarantees. 

%For a vector $p\in \simplex^n$, let $\cat(p)$ denotes the categorical distribution define by $p$. 

\begin{definition}[$\ell_1$-condition number]\label{ass:l_1_condition}
	%We assume that word-topic matrix has bounded $\ell_1$-condition number $\kappa$ in the sense that 
	For a nonnegative matrix $A$, define its $\ell_1$-condition number $\kappa(A)$ to be the minimum $\kappa$ such that for any $x \in \R^k$, 
	\begin{equation}
	\|Ax\|_1 \ge \|x\|_1/\kappa
	\end{equation}
\end{definition}

This condition number was introduced by \citet{KleinbergS08} in analyzing various algorithms for collaborative filtering. We will use a weaker (i.e. smaller) notion of condition number. Empirically, it seems that most of the word-topic matrices that we have encountered have a reasonably small $\ell_1$-condition number, and have an even smaller $\ell_\infty \rightarrow \ell_1$-condition number. 

\begin{definition}[$\ell_\infty \rightarrow \ell_1$-condition number]\label{def:linfty_to_one}
	Let $\lambda(A)$ be the minimum number $\lambda$ such that 
	for any $x \in \R^k$, 
	\begin{equation}
		\|Ax\|_1 \ge \|x\|_{\infty}/\lambda
	\end{equation}
\end{definition}
\begin{remark}
	Based on the relationship between $\ell_1$ and $\ell_\infty$ norm, we have that $\lambda(A) \le \kappa(A) \le k\lambda(A) $. In Section~\ref{sec:lowerbound} we give an example where the $\ell_1\to \ell_1$ condition number is significantly worse: $\kappa(A) \ge \Omega(\sqrt{k}) \lambda(A)$.
\end{remark}
%\begin{remark}
%	Note that it's important to allow $x\in \R^k$ instead of $x\in \R_{\ge 0}^k$ since for the latter we always have $\|Ax\|_1 = \|x\|_1$. %Moreover, we call it $\ell_1$-condition number since $$\kappa = \frac{\max\frac{\|Ax\|_1}{\|x\|_1}}{\max\frac{\|Ax\|_1}{\|x\|_1}1}$$. 
%\end{remark}

%
%\begin{assumption}\label{ass:bounded_away}
%	We assume that $\xstar\in \simplex_k$ satisfies that $R = \supp(\xstar)$ is of size at most $r$ and $\xstar_i \ge \delta/r$ for any $i\in R$. 
%\end{assumption}

\section{$\delta$-Biased Minimum Variance Estimators}

Let $y\in \R^D$ be the document vector whose $i$-th entry $y_i$ is the number of times word $i$ appears. We try to estimate the true topic vector $x^*$ by left multiplying $y$ with some matrix $B$. Intuitively, $\E[By] = BAx^*$, so we want $BA$ to be close to the identity matrix. On the other hand, when we apply $B$ to the document vector, each word will select a column of $B$, and its variance on {\em any} entry is bounded by the maximum entry in $B$. Therefore we would like to optimize over two things: first, we want $BA$ to be close to identity; second, we want the matrix $B$ to have small $\vecinftynorm{B}$. This inspires the following linear program:

\begin{definition}\label{def:delta}
	For $A\in \R^{D\times k}$ and $\delta \ge 0$, define $\lambda_{\delta}(A)$ to be the solution of the following convex program:  
	\begin{align}
	\lambda_{\delta}(A) = \textup{min} \quad &  \vecinftynorm{B} \nonumber\\
	\textup{s.t.} \quad & \vecinftynorm{BA-\Id_{k}} \le \delta \nonumber \\
	& B\in \R^{k\times D}\label{eqn:program}
	\end{align}	
\end{definition}

We will refer to the minimizer $B$ of the above convex program as the $\delta$-biased minimum variance inverse for $A$. %, because when we use it as an estimator we will use its maximum entry to bound its variance. 
The solution to the above convex program will help minimize our sample complexity both theoretically and empirically. 

Allowing a nonzero $\delta$ can potentially reduce the variance of the estimator while introducing a small bias. Such bias-variance trade-off has been studied in other settings~\cite{MS13,JavanmardM14}.

What is the optimal $\vecinftynorm{B}$? To answer this question we get the dual of the LP~\ref{eqn:program} (with variable $Q\in \R^{k\times k}$), 
\begin{align}
	\textup{maximize} \quad & \textup{tr}(Q) - \delta \veconenorm{Q}\nonumber\\
	\textup{s.t.} \quad & |AQ|_1 \le 1 \label{eqn:eqn2}
	%&  $Q\in \R^{} \label{eqn:eqn2}
\end{align}
We can further show that \eqref{eqn:eqn2} is equivalent to the following (non-convex) program with vector variables $x\in \R^k$ (see Appendix~\ref{sec:duality} for the proof): 
	 \begin{align}
	 \textup{maximize} \quad & \|x\|_{\infty} - \delta \|x\|_1 \nonumber\\ 
	 \textup{s.t.} \quad & \|Ax\|_1 \le 1\nonumber 
	 %P+QA^{\top} = 0 \\
	 %&  \veconenorm{P} \le 1\\
	 %	& P_1,P_2,Q_1,Q_2 \ge 0 
	 \end{align}
Note that this is very closely related to the condition number $\lambda$ in Definition~\ref{def:linfty_to_one}. In particular, the optimal value is exactly $\lambda(A)$ when $\delta = 0$! When $\delta > 0$ this can be viewed as a relaxation of the $\ell_\infty\to\ell_1$ condition number. This is summarized in the following Proposition whose proof is deferred to appendix.
%Finally, the following proposition shows that when $\delta = 0$,  optimal value $\lambda_0(A)$ of program~\eqref{eqn:program}  completely characterizes the $\ell_{\infty}\rightarrow \ell_1$-condition number of $A$, which therefore can be computable in polynomial time. When $\delta >0$, the optimal value $\lambda_{\delta}(A)$ could potentially even smaller (which corresponds to lower sample complexity as shown in next section). 
\begin{proposition}\label{prop:condition_number}
	For any $\delta \ge 0$, we have that 
	\begin{equation}
	\lambda_{\delta}(A) \le \lambda_0(A)= \lambda(A)
\le \kappa(A)\,.\nonumber
	\end{equation}
	%Moreover, $\lambda_0(A) = \lambda(A)$. 
\end{proposition}

\section{Recovery Guarantees in the $\ell_1$-Norm}\label{sec:support}

In this section we show how to estimate the topic proportion vector using a $\delta$-biased minimum variance inverse $B$ of word-topic matrix $A$ (Definition~\ref{def:delta}). For a small $\delta$ (that is $\ll 1/r$), given a solution $B$ of program~\eqref{eqn:program} with entries of absolute value at most $\lambda_{\delta}(A)$, the following Thresholded Linear Inverse  estimator (Algorithm~\ref{alg:inference}) is guaranteed to be close to the true $\xstar$ in both $\ell_1$ and $\ell_{\infty}$ norm. 

\begin{algorithm}
\noindent {\bf Input}: Document $y$ with $n$ words, and $\delta$-biased inverse matrix $B$ of matrix $A$.\\
\noindent {\bf Output}: Topic vector estimator $x$.
\begin{enumerate}
\item Compute $\hat{x} =\frac{1}{n} B y$. 
\item Let $\tau = 2\inftybound \sqrt{\log k/n}+\delta$. For all $i\in [k]$, , if $\hat{x}_i <  \tau$, set $x_i = 0$, otherwise set $x_i = \hat{x}_i$. 
\end{enumerate}
\caption{Thresholded Linear Inverse Algorithm  (TLI)\label{alg:inference}}
\end{algorithm}

\begin{theorem}
	\label{thm:maininference}
	%If the original vector $x^*$ is an $r$-sparse vector, 
	Suppose document $y$ is generated from $r$-sparse topic vector $\xstar$.  For any $\epsilon > 4\delta r$, 
	given $n = \Omega(\inftybound^2 r^2 \log k/\epsilon^2)$ samples, with high probability Algorithm~\ref{alg:inference} returns a vector that has $\ell_1$-distance at most $\epsilon$ with $x^*$.
\end{theorem}

Our first step is to bound the variance of the partial estimator $\hat{x}$ before thresholding. Our bound will utilize the maximum entry in $B$, which is why we tried to find $B$ that minimizes this quantity in the first place. %and since we have found (by construction) the matrix that minimizes the infinity norm, %we have found the $B$ that optimizes our bounds. 
In particular we can show:

%In order to prove this Theorem, we first analyze the concentration of an individual coordinate.

\begin{lemma}
\label{lem:entrywise}
With probability at least $1- 1/k^2$, for every $i$ we have $|\hat{x_i} - x^*_i| \le \delta + 2\inftybound\sqrt{(\log k)/n}$.
\end{lemma}
\begin{proof}[Proof of Lemma~\ref{lem:entrywise}]
By definition, $\hat{x}_i = \frac{1}{n} \sum_{j=1}^n (B\one_{w_j})_i$ where $\one_{w_j}$ is the indicator vector for the $j$th word in the document. 
By summing over words in the document as opposed to words in the vocabulary, we have written $\hat{x}_j$ as a sum of 
independent random variables, and we will use Bernstein's inequality to show that it is concentrated around its mean. This is straightforward, but the key is the way we have chosen $B$ ensures that the estimator is at most $\delta$-based. To elaborate, we can compute 
\begin{align*}
	\E[\hat{x}_i] = (BAx^*)_i & = \sum_{j=1}^k (BA)_{i,j}x^*_j \\
	& = x^*_j + \sum_{j=1}^k ((BA)_{i,j} - \mathbf{1}_{i=j}) x^*_j\,,
\end{align*}
where $\mathbf{1}_{i=j} = 1$ if $i=j$ and $\mathbf{1}_{i=j} = 0$ otherwise. Now by construction we have that for all $i$ and $j$, $|(BA)_{i,j} - \mathbf{1}_{i=j}| \le \delta$. Hence, 
$|\sum_{j=1}^k ((BA)_{i,j} - \mathbf{1}_{i=j}) x^*_j| \le \delta \sum_{j=1}^k x^*_j = \delta\mper\nonumber$
Therefore we conclude that $|\E[\hat{x}_i] - x^*_i| \le \delta$ which shows that our estimator has bias at most $\delta$. 

Now we can appeal to standard concentration arguments. 
Recall that $\hat{x}_i$ is a sum of independent random variables $\hat{x}_i = \frac{1}{n} \sum_{j=1}^n (B\one_{w_j})_i$, and each summand here is bounded by $\max (B\one_{w_j})_i \le \inftybound$. We apply
Hoeffding's inequality and obtain that  with probability at least $1-1/k^2$, $|\hat{x}_i - \E[\hat{x}_i]| \le 2\inftybound\sqrt{(\log k)/n}$ and this completes the proof of the lemma. 
\end{proof}

Lemma above shows that the vector $\hat{x}$ is close to the true $x^*$ in infinity norm. As a corollary, we know the algorithm finds the correct support if $x^*$ does not have very small entries

\begin{corollary}\label{cor:finding_correct_support}
With high probability, $x$ output by Algorithm~\ref{alg:inference} satisfies that for every $i\in [k]$, if $x^*_i = 0$ then $x_i = 0$, and if $x^*_i \ge 4\inftybound \sqrt{(\log k)/n}+2\delta$ then $x_i > 0$. In particular, if all the nonzero entries of $x^*$ are at least $\epsilon/r$ for some $\epsilon>4\delta r$, the algorithm finds the correct support with $O(\inftybound^2 r^2\log k /\epsilon^2)$ samples.
\end{corollary}

%When $x^*$ is a sparse vector, the algorithm is very unlikely to pick up a coordinate that is $0$ in $x^*$. 
Using the corollary above we can then prove Theorem~\ref{thm:maininference}. The key intuition is $x$ can only incur error on non-zero coordinates of $\xstar$, and a fixed amount of error on non-zero coordinates of $\xstar$.  

\begin{proof}[Proof of Theorem~\ref{thm:maininference}]
By Lemma~\ref{lem:entrywise} and union bound, we have that with probability at least $1-1/k$, for every $i\in [k]$ $|\hat{x_i} - x^*_i| \le \delta + 2\inftybound\sqrt{(\log k)/n})$. Thus in Step $2$ of the algorithm we are guaranteed that if $x^*_i = 0$ then $\hat{x}_i$ must be smaller than the threshold and therefore $x_i = 0$.

On the other hand, if $x^*_i > 0$, we know $$|x^*_i - x_i| \le |x^*_i - \hat{x}_i| + |\hat{x}_i - x_i| \le 2\delta + 4\inftybound\sqrt{(\log k)/n})$$ Again appealing to the fact that $\hat{x}$ and $x^*$ are entry-wise close, there can be at most $r$ entries where we set $x^*_i > 0$ for $\delta\le \epsilon/4r$. When $n = 64\kappa^2r^2\log k/\epsilon^2$ we also have $4\inftybound\sqrt{(\log k)/n}) \le \epsilon/2r$. Combining these two facts we conclude $|x^*-x|_1 \le r(2\delta + 4\inftybound\sqrt{(\log k)/n})) \le \epsilon$, which completes the proof. 
\end{proof}

\section{Rate of MLE estimator}\label{sec:mle}

%Given an observed document $y$, the posterior distribution $x\mid y$ has density function $p[x\vert y] \propto p(y\vert x)p(x)$. Therefore the log-likelihood function can be written as 

%\Tnote{In this whole section we will work with the support $R$. So current in the notation the subscript $R$ is omitted! }

%\Tnote{all this section requires $x_i \ge \ltau/r$ for if $x_i \neq 0$. }

%\Tnote{add notation: Let $\hat{a}_i = (a_i)_R$}

In this section, we show that given the correct support $R$ of $\xstar$, we can optimize the log-likelihood function over the variables in $R$ and obtain a finer solution with smaller $\ell_1$ error. We make the following two assumptions: first, that the non-zero coordinates of $\xstar$ are bounded away from zero; second, that the word-topic matrix has small restricted $\ell_1\rightarrow \ell_1$ condition number.

\begin{assumption}\label{ass:bounded_away}
	We assume that $\xstar\in \simplex_k$ satisfies that $R = \supp(\xstar)$ is of size at most $r$ and $\xstar_i \ge \ltau/r$ for any $i\in R$. 
\end{assumption}

%\Tnote{Unfortunately the $\ltau$ is this section is different from the $\ltau$ all other section}
%\Tnote{Note that this $\barkappa$ is different from Section 2's $\barkappa$. }

\begin{assumption}[restricted $\ell_1\rightarrow \ell_1$ condition number]\label{ass:restricted_ell_1}
	We assume that word-topic matrix $A$ satisfies that for any $r$-sparse vector $v\in \R^d$, 
	\begin{equation}
		\|Av\|_1 \ge \|v\|_1/\barkappa\mper\nonumber
	\end{equation}
\end{assumption}

We note that by definition $\barkappa \le \kappa(A)$. Moreover, the restricted $\ell_1\rightarrow \ell_1$ condition number can be viewed as $\ell_1$ analog of the restricted isometry property~\cite{CandesTao} or restricted eigenvalue conditions~\cite{bickel2009,meinshausen2009} associated to $\ell_2$ norm. This type of assumption is particularly useful (and somewhat necessary) for the estimation problem. 

We will restricted our attention to support $R$ throughout this section. Let $\hata_w\in \R^r$ be the restriction of $a_w$ to the support $R$, and $\hat{A}$ be the word-topic matrix restricted to columns indexed by $R$.  Let $f(x)$ be the log-likelihood function restricted to the support $R$. That is, for $x\in \R^r$, 
\begin{equation}
f(x) =  \log \Pr[y\mid x] = \sum_{w\in y} \log(\inner{\hata_{w},x}), \mper % - \beta\sum_{i=1}^{k}\log x_i + const
\end{equation}

The main theorem in this section below shows that when $n =\Omega(r^2)$, the maximum likelihood estimator (MLE) restricted to the support $R$ has $\ell_1$ rate $\widetilde{O}(\barkappa\sqrt{r/n})$. Moreover, the error on predicting $A\xstar$ is $\widetilde{O}(\sqrt{r/n})$ which doesn't depend on the condition number. 

\begin{theorem}\label{thm:rate}
	%\Tnote{Missing the condition on $n$}
		Under assumption~\ref{ass:bounded_away} and~\ref{ass:restricted_ell_1}, suppose $n\ge c\barkappa^2r^2\log k/\ltau^2$ for a sufficiently large constant $c$.
	Let $x_{\mle}$ be the maximizer of the log-likelihood function $f(x)$ restricted to support $R$. Then with high probability $x_{\mle}$ satisfies that 
	\begin{equation}
	\|Ax_{\mle} - A\xstar\|_1 \le \widetilde{O}\left(\sqrt{\frac{r}{n}}\right)\,,\nonumber
	\end{equation}
	and, 
	\begin{equation}
	\|x_{\mle} - \xstar\|_1\le \widetilde{O}\left(\barkappa\sqrt{\frac{r}{n}}\right) \mper \nonumber
	\end{equation}
\end{theorem} 
%where $\beta = 1-\alpha\in [0,1]$. %satisfies that $0 \le \beta \le 1$. 

Asymptotically, we know that the error vector $x_{\mle} - \xstar$ converges to standard normal with covariance matrix $Q$ being the Fisher information matrix (see Equation (\ref{eqn:fisher})). This means that $Q^{-1/2}(x_{\mle} - \xstar)$ is bounded in $\ell_2$ norm. Therefore the keys towards proving Theorem~\ref{thm:rate} consists of a) converting the above to a non-asymptotic bound with careful concentration inequality b) understanding how $Q^{-1/2}$ converts $\ell_1$ space to $\ell_2$ space so that an $\ell_1$ norm of the error  can be obtained. 

We give intuitions for the proofs here, which mostly follows from the classical asymptotic normality of Maximum Likelihood Estimator, and our main contribution here is to give a finite sample bound using concentration inequalities.

First we consider the gradients and Hessians of the likelihood function.
\begin{equation}
\nabla f(x) = \sum_{w\in y} \frac{\hat{a}_w}{\inner{\hat{a}_w,x}}, ~\nabla^2 f(x)  = -\sum_{w\in y} \frac{\hata_{w}\hata_{w}^{\top}}{\inner{\hata_{w},x}^2} \mper\label{eqn:hessian}%+ \beta\diag(x_i^{-2})\label{eqn:gradient}
\end{equation}
Let $Q$ be the Fisher information matrix as defined below, 
\begin{equation}Q = \Exp\left[\frac{\hata_{w}\hata_{w}^{\top}}{\inner{\hata_{w},x}^2}\right] =  \sum_{i\in [D]}\inner{\hat{a}_i,x^*} \frac{\hat{a}_i\hat{a}_i^{\top}}{\inner{\hat{a}_i,x^*}^2} \mper\label{eqn:fisher}
\end{equation} 
Note that we have $\Exp[\nabla^2 f(\xstar)] = -nQ$. When $n$ is sufficiently large and $x_{\mle}$ is sufficiently close to $\xstar$, we have, 
\begin{equation}
  -\nabla f(x^*)  = \nabla f(x_{\mle}) - \nabla f(\xstar) \approx \nabla^2 f(\xstar) (x_{\mle}-\xstar)\mper\nonumber
\end{equation}
Therefore, it follows that 
\begin{equation}
\xstar - x_{\mle} \approx \nabla^2 f(\xstar)^{-1} \nabla f(x^*) \mper \nonumber
\end{equation}
It can be shown that the covariance of the gradient is  $\Exp[\nabla f(\xstar)\nabla f(\xstar)^{\top}]= nQ$. Therefore when $\nabla^2 f(\xstar)^{-1}$ is sufficiently close to its expectation $nQ$, the covariance of the error  $\xstar - x_{\mle} $ is approximately equal to 
$\frac{1}{n}Q^{-1} $. 

Towards establishing a non-asymptotic result, 
we bound from above $\nabla f(\xstar)$ and lower from below $\nabla^2 f(x)$ in a proper Euclidean norm -- the norm defined by the Fisher information matrix $Q$ in the following three lemmas. Lemma~\ref{thm:gradient_concentration} below controls the $\ell_2$ and $\ell_{\infty}$ norm of $Q^{-1/2}\nabla f(\xstar)$ (which is supposed to be spherical Gaussian with covariance matrix $\sqrt{n}\Id_r$ asymptotically).  
\begin{lemma}\label{thm:gradient_concentration}
		Under assumption~\ref{ass:bounded_away} and~\ref{ass:restricted_ell_1}, suppose $n\ge c\barkappa^2r^2/\ltau^2\cdot \log k$ % \Tnote{To check this bound}
		for a sufficiently large constant $c$. Then, with high probability we have 
		\begin{equation}
		\|Q^{-1/2}(\nabla f(\xstar)- n\one_r)\| \le \widetilde{O}(\sqrt{nr})\,,
		\end{equation}
		and, 
		\begin{equation}
			\|Q^{-1/2}(\nabla f(\xstar)- n\one_r)\|_{\infty} \le \widetilde{O}(\sqrt{n})\mper \label{eqn:gradient_inftynorm}
		\end{equation}
\end{lemma}
Lemma~\ref{thm:lower_bound_M} relates $-\nabla^2 f(x)$ with the Fisher information matrix $Q$ spectrally around a neighborhood of $\xstar$ which MLE will be proved to fall in. We essentially show that $-\nabla f(\xstar)$ concentrates around its expectation $nQ$, and moreover we can effectively approximate the Hessian $-\nabla^2 f(x)$ around a neighborhood of $\xstar$ by $nQ$ as well. 
\begin{lemma}\label{thm:lower_bound_M}
  	Under assumption~\ref{ass:bounded_away} and~\ref{ass:restricted_ell_1}, suppose $n  =  c_0\barkappa^2r^2\log k/\ltau^2$ for sufficiently large constant $c_0$. %Let $M_R(x)$ denote the restriction of $M(x)$ to the indices $R\times R$. 
  	Then, with high probability over the randomness of $y$, it holds that for all $x$ such that $ x\le Cx^*$, \begin{equation}
  	-\nabla^2 f(x)\succeq \frac{n}{2C^2}\cdot Q\,.\label{eqn:psd_lower_bound}
  	\end{equation}
\end{lemma}
Finally, as alluded before, Lemma~\ref{lem:convexity_expectation_restricted} characterizes the distortion caused by $Q^{-1/2}$ transforming $\ell_2$ to $\ell_1$ space.  We note that the square-root of Fisher information matrix $Q^{1/2}$ converts naturally $\ell_1$ to $\ell_2$. Therefore we can get the desired $\ell_1$ error bound in Theorem~\ref{thm:rate}. %and throughout this section, 
%  and the general takeaway of the technique in this section is that we should apply linear transformation $Q^{-1/2}$ on both Hessian and gradient, and then we are in Euclidean space and concentration inequalities can be easily applied.
\begin{lemma}\label{lem:convexity_expectation_restricted}
	%sFor any $x \in \simplex_k$ with $x \le Cx^*$, we have that 
	%The expected value of $\nabla^2 f(x^*)$ can be lowerbouned by 
		Under assumption~\ref{ass:bounded_away} and~\ref{ass:restricted_ell_1}, Fisher information matrix $Q$ satisfies that 
	\begin{equation}
	%\|Q^{-1/2}\|_{2\rightarrow 2}\le
	\|\hat{A}Q^{-1/2}\|_{2\rightarrow 1} \le 1, \quad \textrm{and} \quad \|Q^{-1/2}\|_{2\rightarrow 1} \le \barkappa\mper
	\end{equation}
	
	%	\begin{equation}
	%		\Exp[M(x^*)] = n \sum_{i\in D}  \inner{a_i, x^*} \cdot\frac{a_{i}a_{i}^{\top}}{\inner{a_{i},x}^2}\succeq \frac{n}{C\barkappa^2} \cdot \Id_k
	%	\end{equation}
	As a corollary, 
	\begin{equation}
	%\Exp[\nabla^2 f(\xstar)] = 
	Q \succeq\frac{1}{\barkappa^2} \cdot \Id_r\mper\label{eqn:lem:expectation}
	\end{equation}
	%ewhere the expectation is over the randomness of the generating process of the document. 
\end{lemma}

 %Therefore after applying the whitening with \ell_2 

%The proof of Lemmas above are deferred to Appendix~\ref{sec:app:mle}. With these lemmas we are ready to prove Theorem~\ref{thm:rate}. 

\section{Sample Complexity Lower Bounds}\label{sec:lowerbound}

In this section we construct a natural (distribution of) word-topic matrix $A$ with low $\Lambda_\delta(A)$ value for very small $\delta$ for which given document with $o(r^2)$ words,  it is impossible to determine the support $\xstar$ %of an $r$-sparse vector 
even if all the nonzero coordinates of $\xstar$ are roughly $1/r$. This shows that for the task of support recovery, our algorithm in Section~\ref{sec:support} achieves optimal sample complexity up to logarithmic factor. 
%This shows the sample complexity analysis in Section~\ref{sec:support} is tight.

\begin{theorem}
	There exists a (distribution of) word-topic matrix $A$ with $\Lambda_{\delta}(A) = 1$ for $\delta = O(\sqrt{\log k/D})$ such that any algorithm $\mathcal A$ that takes document of $o(r^2)$ words as input cannot recover the support of the topic vector $\xstar$ that is used to generate $y$ with probability $3/4$. This is still true when $\xstar$ is promised to have only non-zero entries that are larger than $1/r$. 
\end{theorem}

%To construct the matrix $A$, 
The hard instance that we constructed is pretty simple: we consider a word-topic matrix  where every topic contains roughly  half of words in the vocabulary, and gives probability roughly $2/D$ to each of the words. The words in the topic are uniformly randomly selected. To make this more precise, let $S_1, S_2,..., S_k \subset[D]$ be $k$ independent subsets that are uniformly chosen among all subsets of $[D]$. Let matrix $A_{i,j} = 1/|S_j|$ if $i\in S_j$ and $A_{i,j} = 0$ otherwise.

We first show that indeed $\Lambda_{\delta}(A)$ is very small, and therefore it has a good $\delta$-biased minimum variance estimator and our algorithm works well on this matrix. % falls into the general consideration of this paper. %The proof is deferred to Section~\ref{sec:app:lowerbound}.  %and our linear inverse algorithm is very effective for this matrix $A$.

\begin{lemma}\label{lem:inftyonecondition}
With high probability over the randomness of $A$, we have $1-\delta \le \Lambda_{\delta}(A) \le 1$ for any $\delta \ge c\sqrt{(\log k)/D}$ where $c$ is a sufficiently large constant.
\end{lemma}

\begin{proof}[Proof of Lemma~\ref{lem:inftyonecondition}] For any matrix $A$ with columns having $\ell_1$ norm 1, we have that $\Lambda_{\delta}(A)\ge 1-\delta$. We show the upper bound by constructing the linear inverse matrix $B$ with small $\vecinftynorm{B}\le 1$ explicitly: 
	%We can directly construct the linear inverse matrix $B$:
	$B_{j,i} = 1$ if $A_{i,j} > 0$, and $B_{j,i} = -1$ if $A_{i,j} = 0$. Let $b_i$ be the $i$-th row of $B$ and $a_i$ be the $i$-th column of $A$, to verify $\vecinftynorm{BA - I} \le \delta$, it suffices to show that
1) $\inner{b_i,a_i} = 1$ for all $i$, 
		and 2) $|\inner{b_i,a_j}| \le \delta$ for all $i\ne j$.

	The first equation is easy because $b_i$ is $1$ on the support of $a_i$, so $\inner{b_i,a_i} = |a_i|_1 = 1$.
	
	For the second equation, consider an arbitrary pair $b_i, a_j$ ($i\ne j$). Consider $S_i,S_j$, the supports of $a_i$ and $a_j$ respectively. By the construction of $A,B$ we know that all entries in $S_i\cap S_j$ contribute $+1/|S_j|$ to the inner product $\inner{b_i,a_j}$, that all entries in $S_j\backslash S_i$ contribute $-1/|S_j|$ to the inner product, and that all other entries contribute 0. Therefore, 
	$\inner{b_i, a_j} = \frac{1}{|S_j|}(|S_i\cap S_j| - |S_j\backslash S_i|)\mper$
	
	Standard concentration bounds show that $|S_i\cap S_j|, |S_j\backslash S_i|$ are within $D/4 \pm O(\sqrt{D\log k})$, and $|S_j|$ is within $D/2 \pm O(\sqrt{D\log k})$ with probability at least $1-1/k^4$. Therefore by union bound with probability at least $1-1/k^2$ for all pairs $i,j$ we have $|\inner{b_i, a_j}| \le O(\sqrt{(\log k)/D})$. 
\end{proof}

Lemma~\ref{lem:inftyonecondition} in particular implies that if nonzero entries in the true topic vector $x^*$ are at least $1/r$, our algorithm can detect the support with $O(r^2\log k)$ samples. We show below that no algorithm can do much better by constructing the following hard instance:

\begin{lemma}\label{lem:lower_bound}
With high probability over the choice of $A$, there exists vector $r$-sparse vectors $x,x^{-}$ with non-zero entries bounded below from $1/r$, such that 
%and the choices of $R$ and $R^-$,
 no algorithm that only takes document of $o(r^2)$ words can distinguish distributions $\cat(Ax)$ and $\cat(Ax^-)$ with probability better than $1/2+o(1)$.
\end{lemma}

The full proof is deferred to Section~\ref{sec:app:lowerbound}. Here we sketch the main idea. We construct $x,x^{-}$  randomly as follows. Let $R \subset [k]$ be a random subset of size $r$, and let $R^- \subset R$ be a random subset with one item removed from $R$ (so $|R^-| = r-1$). Let $x_i = 1/r$ if $i\in R$, and $x_i = 0$ otherwise. Let $x^-_i = 1/(r-1)$ if $i\in R^-$ and 0 otherwise. The key here is to show that with high probability over the choice of $x,x^{-}$ and the choice of $A$, the KL-divergence between the two distribution $\cat(Ax)$ and $\cat(Ax^{-})$ is $O(1/r^2)$, and therefore $\Omega(r^2)$ is need for distinguishing them. 

Let $p = Ax$ and $q = Ax^{-}$. We will show that $p_j/q_j$ is between $1\pm O(1/r)$ for most of the word $j$, and the rest of words are negligible. To see this, we observe that for most of the words, $p_j\ge \Omega(1/D)$ , and $|p_j - q_j| = |(Ax)_j - (Ax^{-})_j|\le O(1/r) \max_i A_{j,i} = O(1/(rD))$. Finally, we can bound KL-divergence of $p,q$ by $O(1/r^2)$,  
$
\textup{KL}(p\| q)\le \chi^{2}(p,q) = \sum_j (p_j-q_j)^2/q_j\le O(1/r^2)$, as desired.

In fact, using the same matrix we can bound the difference between $\ell_1\to\ell_1$ condition number $\kappa(A)$ and $\ell_1 \to \ell_\infty$ condition number $\lambda(A)$.

\begin{lemma}\label{lem:difference}
When $D\gg k\log k$, with high probability, the $\ell_1\to\ell_1$ condition number $\kappa(A) \ge \Omega(\sqrt{k})$. When $D \gg k^2 \log k$, with high probability $\lambda(A) \le 2$.
\end{lemma}

We give the proof in Appendix~\ref{sec:app:lowerbound}. Intuitively, for $\kappa(A)$ we show that the uniform mixture of first half of topics is very similar to the uniform mixture of the last half of topics. For $\lambda(A)$, we use the construction for the $\delta$-biased linear inverse, and show that ``fixing'' the bias does not change the condition number by too much.

\section{Experiments}\label{sec:experiments}
The corpora we use consist of New York Times articles (295,000 documents, average document length 298), Enron emails (39,861 documents, 	average document length 136), and NIPS papers  (1500 documents, average length 1042).
We compute the word-topic matrices using the algorithm in~\cite{AroraGHMMSWZ13}, using 100 topics for NIPS, 100 topics for Enron,
and 400 topics for NYTimes.

\paragraph{Condition Numbers of Matrices}

First we empirically verify the assumption that the word-topic matrix have small $\ell_{\infty}\to \ell_1$ condition number (see Table~\ref{tab:condition_number}). 
%have good $\delta$-biased minimum variance estimators. See Table~\ref{tab:condition_number}. 
Solving LP (\ref{eqn:program}) on 16 processors in parallel using the Mosek LP solver takes 1 minute and a half for the NIPS dataset, 4 minutes for the Enron dataset, and roughly 4 hours for the NYTimes dataset (this is partly the result of using more topics for this dataset).

Note that we only need to compute the inverse matrix {\em once} and then it can be used to do inference on all the documents, so the running time for computing the inverse is not a major concern. The procedure can also be easily parallelized because the LP for different rows of $B$ are independent.

%\authnote{TM}{below needs to be better explained and probably moved to other place}
For comparison, we also list lower bounds on the $\ell_1 \to \ell_1$ condition number $\kappa(\cdot)$, the condition used by~\cite{KleinbergS08}. We see that it's at least 2 times larger than $\ell_{\infty}\to\ell_1$ condition number. %Computing $\ell_1\to\ell_1$ condition number is intractable, 
There is no efficient algorithm known for computing $\ell_1\to\ell_1$ condition number, and therefore we only compute provable lower bounds for various dataset (see Section~\ref{sec:condition_number} for the approach). 
%\Tnote{added paragraph above and moved the math to appendix}
%In Table~\ref{tab:condition_number} also list lower bounds on the $\ell_1 \to \ell_1$ condition number, $\kappa_1(A)$. %which are derived from the $\lambda_{\delta}(A)$ values as follows. 
%In each of the datasets we get the given bound by considering $\delta = 0.001$.

%[Table 1: columns: data sets. row 1: number of words, row 2: number of topics, row 3: condition number, row 4: condition number with bias
%row 5:running time]

%By allowing a small bias, we make a constant factor improvements to the condition number. This allows the algorithm to work with fewer words per document.
%
%\authnote{FK}{As far as I can tell above claim isn't really true. Plan to remove. Note I also changed table to just list $\lambda_{0.01}(A)$ values which is what really matters. I also added description of the not-matching-theory thresholding scheme in synthetic experiments section, maybe should be edited.}

%(Maybe: Another observations is different rows of $B$ matrix can have different infinity norm, and for many rows the infinity norm is even smaller. and just give a basic example).

\paragraph{Synthetic Experiments} We first verify the recovery guarantee of our algorithm on synthetic documents. For each document, we sample $r = 5$ topics uniformly at random, and choose weights for these topics uniformly from the $r$-dimensional probability simplex. The results\footnote{Code to reproduce the results is available at: \url{https://github.com/frytvm/topic-inference}}
 are listed in Figure~\ref{fig:recovery_error2}.

\begin{table*}
	%\captionsetup{font=small}
	\centering
	{
		\begin{tabular}{c|ccccccc}
			\hline 
			& \#~words& \#~topics& $\lambda(A)$ & $\lambda_{0.001}(A)$ & $\lambda_{0.01}(A)$ & $\lambda_{0.1}(A)$ & $\kappa(A)$\\
			\hline\hline
			NIPS & 3.9k & 100& 1.547 & 1.544 & 1.515 & 1.245 & $\ge 3.334$ \\
			Enron & 5.8k & 100 & 5.032 & 5.019 & 4.908 & 3.877 & $\ge 12.46$ \\
			NYTimes & 9.5k & 400 & 2.990 & 2.984 & 2.923 & 2.349 & $\ge 6.755$ \\
			\hline
		\end{tabular}

\iffalse
	\begin{tabular}{c|ccc}
			\hline 
			& \#~words& \#~topic& $\lambda_{0.01}(A)$ \\
			\hline\hline
			NIPS & 3.9k & 100& 1.55\\
			Enron & 5.8k & 100 & 4.91\\
			NYTimes & 9.5k & 400 & 2.92 \\
			\hline
		\end{tabular}
\fi
	}
	\caption{Condition numbers $\lambda_{\delta}(A)$ and $\kappa(A)$ of word-topic matrices trained from datasets (smaller is better, always $\ge 1$)}
	\label{tab:condition_number}
\end{table*}

\begin{figure*}
\begin{center}
\includegraphics[trim= 0 0 120 0,clip,scale=0.50]{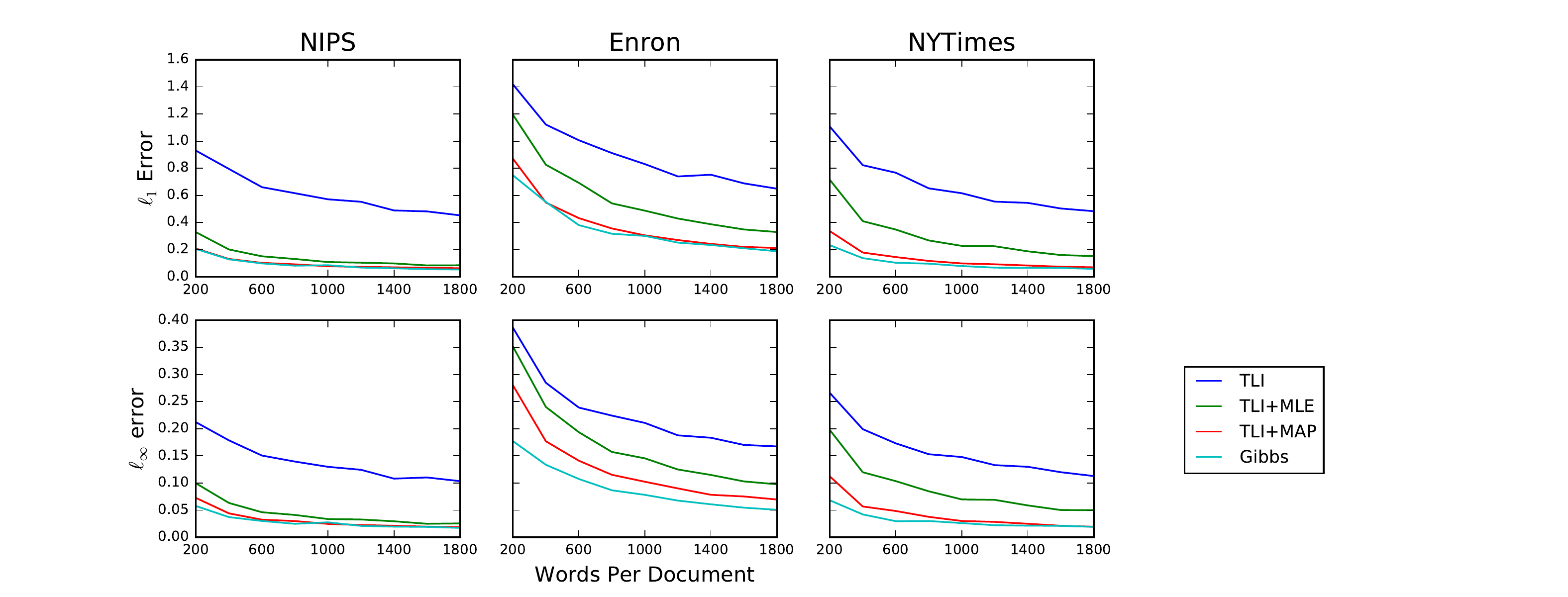}
\end{center}
\caption{Estimation error on semi-synthetic data in $\ell_1$ and $\ell_{\infty}$ norms: uniform sparse topics, $r = 5$ }
\label{fig:recovery_error2}
\end{figure*}

\begin{figure*}
\begin{center}
\includegraphics[trim= 0 0 120 0,clip,scale=0.50]{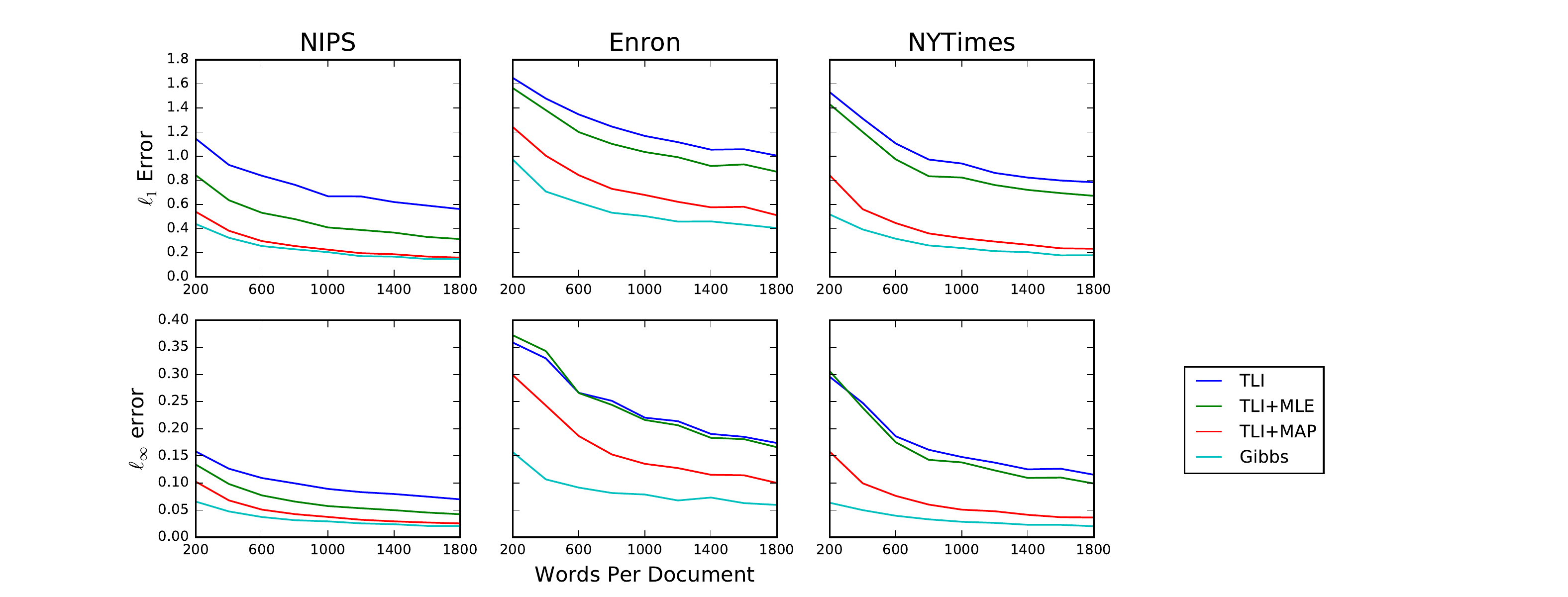}
\end{center}
\caption{Estimation error on semi-synthetic data in $\ell_1$ and $\ell_{\infty}$ norms: Dirichlet prior, $\alpha = r/k = 5/k$}
\label{fig:recovery_error_dirichlet}
\end{figure*}

\begin{figure*}
\begin{center}
\includegraphics[trim= 0 0 120 0,clip,scale=0.50]{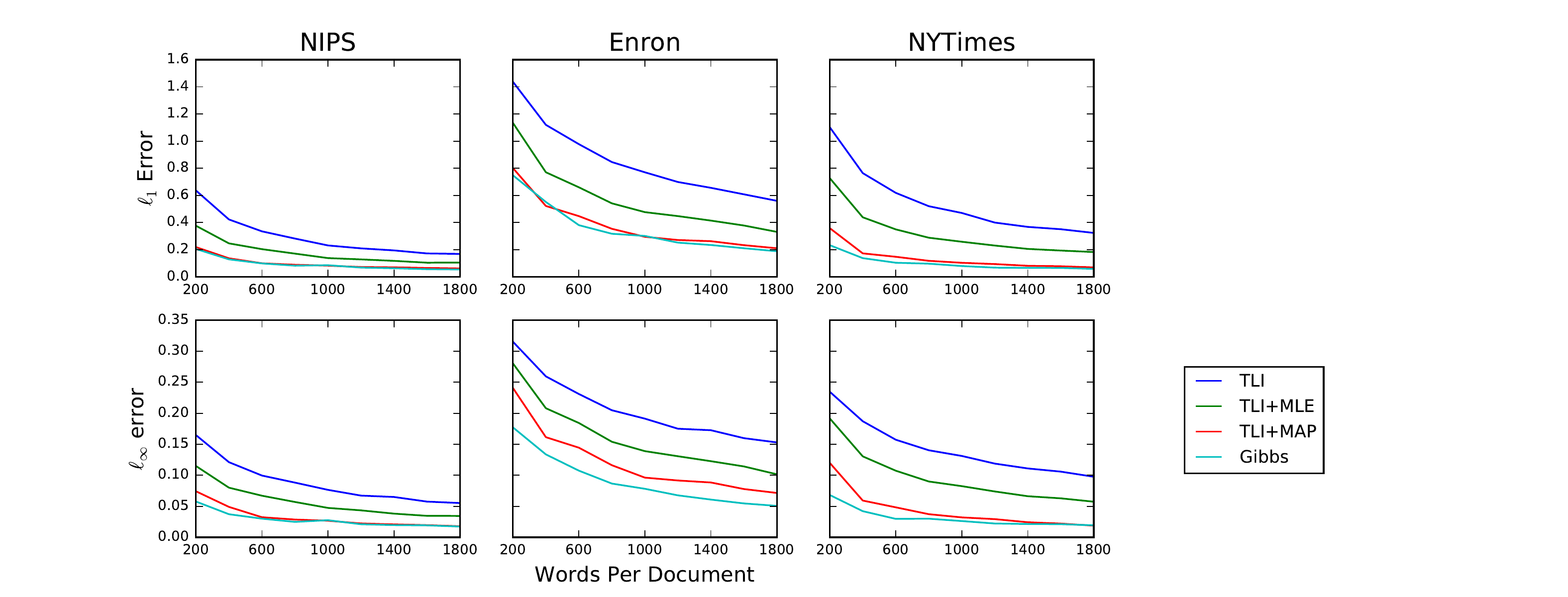}
\end{center}
\caption{Estimation error on semi-synthetic data in $\ell_1$ and $\ell_{\infty}$ norms: uniform sparse topics, $r = 5$, \emph{using top-r truncation for TLI}}
\label{fig:recovery_error}
\end{figure*}

\begin{figure*}
\begin{center}
\advance\leftskip2cm
\includegraphics[trim= 0 0 0 0,clip,scale=0.50]{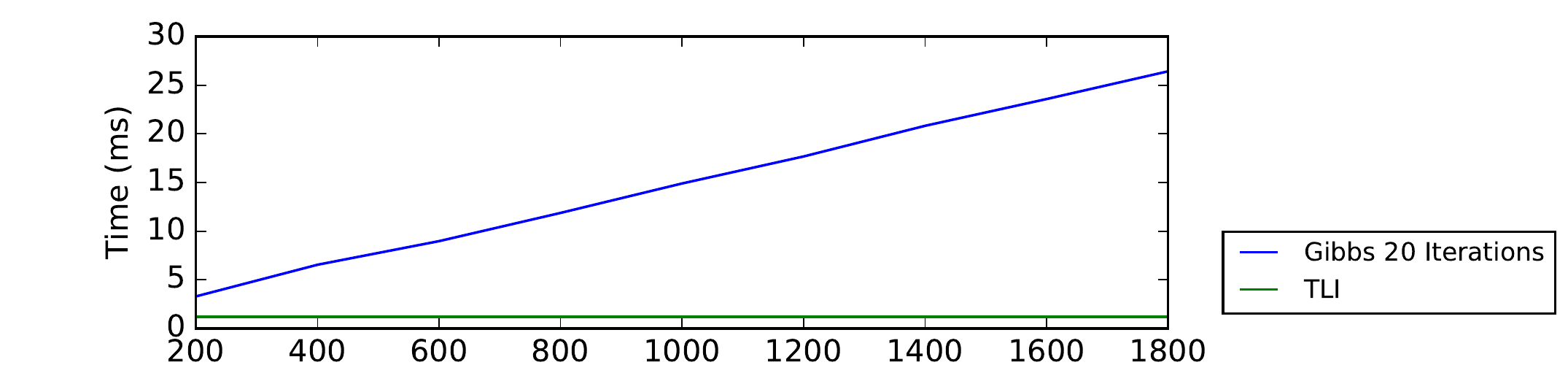}
\end{center}
\caption{Speed of TLI vs Gibbs Inference (20 iterations)}
\label{fig:speed}
\end{figure*}

%[Figure 1: 2*3 plots, columns for data sets, rows for $\ell_1$/$\ell_\infty$, $x$ axis = words, $y$ axis = error. Plots for our algorithm, our algorithm with %Maximum Likelihood improvement, Gibbs sampling with Dirichlet prior, and maybe our algorithm without the bias]

%It must be noted that the suggested thresholding values used in the TLI algorithm above are for the most part too agressive (i.e. the thresholding value is too large) to work well in practice. Therefore we must consider alternative thresholding schemes; in the experiments below we use a top-$r$ scheme where we drop all but the top-$r$ values. We abuse notation and refer to this also as TLI. 
We compare our TLI algorithm (Algorithm~\ref{alg:inference})\footnote{For documents of the length found in the corpuses, the thresholding value $\tau$ used in the theoretical section is too conservative (large). As a more practical alternative we replace $\tau$ as given in the theoretical section by $\tau/4.5$. We use unbiased pseudoinverses, taking $\delta = 0$. }
 with the collapsed Gibbs sampling algorithm implemented in MALLET~\cite{MALLET} and its anchor word compatible extension\footnote{https://github.com/mimno/anchor}; we use 200 iteration of burn-in and 1000 further iterations of sampling. Note that when applying Gibbs sampling for the topic vectors, we treat the word-topic matrix as a fixed constant.
 It is not easy to do Gibbs Sampling for the prior we specified, so we compare against a Dirichlet prior with $\alpha = \frac{r}{k}\mathbf{1}$ which encourages $r$-sparse vectors.
 Note that TLI-Unnormalized is the output of the TLI algorithm as described, whereas TLI is the result after normalizing the entries to give a probability distribution. We can also improve the quality of the TLI estimate by using gradient ascent on the likelihood (see Section~\ref{sec:mle}),
restricted to the top $r$ entries of our initial estimate; this is denoted TLI+MLE in the figure. Finally, we give the result of gradient ascent on the posterior (treating the prior as Dirichlet, similarly to Gibbs sampling) starting from the TLI estimate, and denote this TLI+MAP.

We can also replace the uniform sparse prior with a Dirichlet prior with $\alpha=\frac{r}{k} \mathbf{1}$ and get similar results; see Figure~\ref{fig:recovery_error_dirichlet}.

When using a uniform sparse prior, we can improve the recovery score of TLI by replacing the thresholding step in algorithm TLI with one that drops all but the top-$r$ values; see Figure~\ref{fig:recovery_error}.

The performance of the TLI algorithm is 3 to 5 times worse than Gibbs sampling in terms of $\ell_1$ or $\ell_\infty$ error with same number of words. This is mostly because a simple linear estimator cannot capture the correlations between weights of different topics. 
The post processing using maximum likelihood improves the performance significantly. The performance of the algorithms seems to be related to the $\ell_\infty \to \ell_1$ condition number (with NIPS being best and Enron being worst).

%We also compare our algorithm with the original collaborative filtering algorithm from \cite{KS}.

%\authnote{RG}{Hopefully we can say our result is comparable/better than Gibbs sampling, and also better than the inverse without bias. Maximum likelihood improves the guarantee.}

A virtue of our algorithm is its speed. On the NYTimes dataset, computing the TLI estimate for a single document, which is just a matrix multiplication and a single thresholding step, takes approximately 0.8 milliseconds. In contrast, on a document of length $1600$, a single iteration of Gibbs sampling takes approximately 1.0 ms (and to get the result in the plot we used 1000 iterations). On the Enron semi-synthetic data with uniform sparse prior and documents of length $1600$, we find it takes about 20 total iterations of Gibbs sampling (15 of them as burn in) to return a result of similar accuracy to TLI. The speed of these methods for different length documents is illustrated in Figure~\ref{fig:speed}.
 %So when dealing with long documents and when maximum accuracy is not required, the TLI method may be useful.
%much faster than MALLET: for ?? data sets and ?? words MALLET takes 5 seconds and ours takes 0.1? seconds.

%We also attempted experiments using synthetic documents generated with a Dirichlet distribution on topics; in this case, the $\ell_{\infty}$ errorof TLI is only slightly worse, but the $\ell_1$ error is significantly worse: for the most favorable dataset, NIPS, with 1800 word documents, the $\ell_1$ error is the quite-high $.51$. 
%\authnote{FK}{Above line should probably be edited, but we need to mention Dirichlet experiments or it would be strange.}

\paragraph{Inference on Real Documents}%\authnote{FK}{Rerun with unbiased estimators, add MAP result.}  % MAP Result may not be so interesting...
We run both our algorithm TLI and Gibbs sampling on a subsample of real documents and test the similarity of results. See Table~\ref{tbl:real_data}. Since we don't have the ground truth in this setting, we focus on trying to recover a small number of high-weight topics: we take the top-3 scoring topics from each estimated topic vector, and then the overlap is the cardinality of the intersection divided by $3$. If we treat the Gibbs sampling result as the ``ground truth'', this gives the fraction of the high-probability topics our algorithm correctly finds. We also observe that by taking 5 instead of 3 topics from TLI, we can improve recall (fraction of Gibbs topics found by TLI) at the expense of precision (fraction of TLI topics that are also from Gibbs)\footnote{We only list recall values because in this setting $\mbox{precision} = 3/5 \cdot \mbox{recall}$.}. %Since the version of TLI with additional gradient ascent does not change the initial support, we do not need to test this version seperately.

\iffalse
Since we don't have the ground truth in this setting, we focus on trying to recover topics that are responsible for more than 10\% of words, and show that our algorithm agrees with Gibbs sampling. For each document $i$, let $S_i$ and $T_i$ be the set of topics found by our algorithm and Gibbs sampling, we define precision $p$ and recall $q$ to be

$$
p = \frac{\sum_{i} |S_i\cap T_i|}{\sum_{i} |S_i|};
q = \frac{\sum_{i} |S_i\cap T_i|}{\sum_{i} |T_i|}.
$$

Intuitively, if we treat the result of Gibbs sampling as the ``ground truth'', precision is the fraction of topics our algorithm finds that are correct; recall is the fraction of correct topics that our algorithm is able to find.
\fi

\begin{table}
\begin{center}
\begin{tabular}{c|ccc}
\hline
Corpus & NYTimes & Enron & NIPS  \\
\hline
Overlap: 3 vs. 3 &
34\% &
26\% &
60\% \\
Recall: 5 vs. 3 & 42\% & 40\% & 75\% \\
%Precision: 5 vs. 3 & 29\% & 24\% & 46\% \\
\hline
\end{tabular}
\end{center}
\caption{Results for the top-3 topic recovery experiment on real data, averaged over sample of size 200. In $5$ vs. $3$ experiment we take the top $5$ topics from TLI and compare to the top 3 from Gibbs, which improves recall (cardinality of intersection over cardinality of Gibbs topics).}
\label{tbl:real_data}
\end{table}
%[Table 2: columns = data sets. row 1 = precision row 2 = recall]

%\authnote{RG}{This should be done on a subsampled set of documents. We may want to tune the threshold for our algorithm in order to get the best performance}

%\paragraph{Held-out Likelihood} A popular evaluation technique for topic modelling is the held-out likelihood. For each document, we randomly separate the words into a training set that contains 80\% words, and a test that contains the rest. Run the inference algorithm on the training set, and compare the held-out likelihood defined as
%
%$$
%H = \frac{1}{|test|} \sum_{w\in test} \log \max\{(Ax)_w,1/10k\}.
%$$
%
%Here $1/10k$ is a smoothing to avoid infinities. We compute the held-out likelihood using our estimator with $r = ?$, followed by Maximum Likelihood optimization on the training set. As a comparison, we use MALLET with $\alpha = \frac{r}{k}\mathbb{1}$.
%
%[Table here comparing the held-out likelihood for 3 data sets. I feel we are unlikely to do well in this test.
%Also we may need to choose the appropriate $r$ for both our algorithm and Gibbs sampling so it is time consuming.]
%

\section{Conclusion}

This work takes a step towards designing algorithms with provable guarantees for inference in topic modeling, building upon earlier work of Kleinberg and Sandler~\cite{KleinbergS08} in collaborative filtering. We use a notion of the approximate inverse of a topic matrix %\authnote{FK}{This is confusing, because $\ell_{\infty} \to \ell_1$ condition number only exactly characterizes the behavior when we choose an unbiased inverse matrix; also choosing a biased matrix doesn't help much. Did Kleinberg and Sandler consider thresholding? If not we should instead emphasize the importance of thresholding to produce a better biased estimator.}
(as opposed to its exact inverse) and characterize the mathematical condition
--- namely, the  $\ell_\infty \to \ell_1$ condition number --- that determines how well it behaves as an estimator.  % to construct an efficient thresholded linear estimator. 
Furthermore, we showed that this 
%This is the first algorithm that is guaranteed to give reasonable estimations when the number of 
algorithm approximately solves inference in a document with as few as $O(r^2\log k)$ words where 
$k$ is the number of topics and $r$ is the sparsity of the topic vector generating the document.
%ew as words in a document is as small as $\tilde{O}(r^2)$. 
We have showed that such guarantees are optimal in the sense that there are word-topic matrices for which it is information theoretically impossible to make meaningful conclusions about the topic vector with fewer than $r^2$ samples. We also show that our linear estimator identifies a reasonable set of topics, which allows us to solve (via convex programming) the maximum likelihood problem restricted to this set of topics and get better estimations in theory and practice. We also find that
in practice, the standard pseudoinverse of the topic matrix is a good choice for $B$, though we do not have theory to support it.

The experiments show that topic model matrices associated with real-life corpora have good 
$\ell_\infty \to \ell_1$ condition number, and that the above method works well with synthetic documents generated using these topic matrices. Moreover the running time is comparable to a single iteration of Gibbs sampling. Topic recovery on real-life documents seems to be slightly weaker, and seems to require further modifications to be more robust to model-misspecification. 
\iffalse
The work can be seen as a first step towards understanding the inference problem and designing inference algorithms with provable guarantees. 
%Empirically our estimator achieves a constant-factor approximation to the traditional Bayesian approach while being much faster. We 
We hope this will inspire more research towards provable guarantees for inference algorithms and eventually lead to better inference algorithms in practice.
\fi
\paragraph{Acknowledgments. }
Sanjeev Arora would like to thank the support in part by NSF grants CCF-1527371, DMS-1317308, Simons Investigator Award, Simons Collaboration Grant, and ONR- N00014-16-1-2329.  Tengyu Ma would like to thank the support in part by Sanjeev Arora's grants, Simons Award in Theoretical Computer Science and IBM PhD Fellowship. Ankur Moitra would like to thank the support in part by NSF CAREER Award CCF-1453261, a grant from the MIT NEC Corporation and a Google Faculty Research Award.
\bibliography{ref}
\bibliographystyle{icml2016}

\appendix
\section{Missing proofs in Section~\ref{sec:mle}}\label{sec:app:mle}

\subsection{Proof of Main Theorem}

%Before proving the main theorem, we first state the following lemma which shows the Hessian for every point that is close to $x^*$ is similar to the Hessian at $x^*$, the proof is deferred to the next subsection.
%
%\begin{lemma}\label{thm:lower_bound_M}
%	Under assumption~\ref{ass:bounded_away} and~\ref{ass:restricted_ell_1}, suppose $n  =  c_0\barkappa^2r^2\log k/\ltau^2$ for sufficiently large constant $c_0$. %Let $M_R(x)$ denote the restriction of $M(x)$ to the indices $R\times R$. 
%	Then, with high probability over the randomness of $y$, it holds that for all $x$ such that $ x\le Cx^*$, \begin{equation}
%	-\nabla^2 f(x)\succeq \frac{n}{2C^2}\cdot Q\,.\label{eqn:psd_lower_bound}
%	\end{equation}
%\end{lemma}

We start by claiming that we only need to consider a reasonable neighborhood of $\xstar$ which contains $x_{\mle}$.

\begin{claim}\label{claim:2} Under the setting of Theorem~\ref{thm:rate}, we have $x_{\mle} \le 2\xstar$. 
	\end{claim}
%\Tnote{Failure attempts started here}
%{\bf Claim: } We start by claiming that $x_{\mle} \le 2\xstar$. This is a rather technical claim so readers might consider skip it. 
\begin{proof}[Proof of Claim~\ref{claim:2}]
	For the sake of contradiction, assume this is not true. Then since $\xstar_i \ge \ltau/r$, we have that $x_{\mle}\not\in \mathcal B$ where $\mathcal B = B(\xstar,\ltau/r, \|\cdot \|_{\infty})$ be the $\ltau/r$-infinity norm ball around $\xstar$.  For simplicity we do not write the projection to the orthogonal subspace of $\mathbf{1}$, but all the gradients will be in that subspace while $\ell_1/\ell_\infty$ norms are measured in original space.

	Let $u = \xstar - x_{\mle}$. 
	%Let $\diffmle =z^* -z_{\mle}$. 
	By simple integration along the line of $\diffmle$, we have 
	\begin{align}
	\nabla f(x^*) & = \nabla f(x^*) - \nabla f(x_{\mle}) \nonumber\\
	&= \left(\int_{0}^1 \nabla^2 f(t\diffmle+x_{\mle})dt \right) \diffmle= -H\diffmle\nonumber
	\end{align}
	where $H = -\left(\int_{0}^1 \nabla^2 f(t\diffmle)dt \right)$. Let $H_t = -\nabla^2 f(t\diffmle)$, and therefore $H = \int_0^1 H_t dt$. 
	
	Let $x = su+x_{\mle}$ be the point in $\mathcal B$ with maximum $s$. Therefore we have that $x$ is on the boundary of $\mathcal B$ and therefore $\|x-\xstar\|_{\infty}= \ltau/r$. Then for any $1\ge t\ge s$, we claim that $H_s$ satisfies that $v^{\top}H_tv \ge \Omega(n/\barkappa^2) \|v\|_1^2$ for any vector $v$. Indeed, let $x = tu+x_{\mle}$, 
	we can verify that 
	\begin{align}
	v^{\top} H_t v &\ge \frac{1}{4} \cdot v^{\top}H_1v \ge \frac{n}{8} v^{\top}Qv \ge \Omega(n/\barkappa^2) \|v\|_1^2\mper\nonumber%v^{\top}\left(\sum_{w\in y} \frac{a_wa_w^{\top}}{\inner{a_w,x}^2}\right)v =  \sum_{w\in y} \frac{\inner{a_w,v}^2}{\inner{a_w,x}^2} \ge \\
	%& \ge (\sum_{w\in y}\inner{a_w,x})^{-1}
	\end{align} 
	
	Therefore we obtain that 
	\begin{align}
	v^{\top} H v &= \int_{0}^1 v^{\top }H_t v dt  \nonumber\\
	& \ge \int_{s}^1 \Omega(n/\barkappa^2)\|v\|_1^2 dt = \Omega((1-s)n/\barkappa^2)\|v\|_1^2 \mper\nonumber
	\end{align}
	
	Therefore we obtain that $\|H^{-1/2}\|_{2\rightarrow 1} \le O\left(\sqrt{\frac{\barkappa^2}{(1-s)n}}\right)$.  Similarly we have that since $H_t \succeq \Omega(n)\cdot Q$ for $1\ge t\ge s$,  and $H_t\succeq 0$ for any $t\in [0,1]$ , we have that $H = \int_0^t H_tdt \succeq (1-s)nQ$ and therefore $\|H^{-1/2}Q^{1/2}\|_{2\rightarrow 2}\le O\left(\frac{1}{\sqrt{(1-s)n}}\right)$. Then we have that 
	\begin{align}
	\|H^{-1}Q^{1/2}\|_{2\rightarrow 1}& \le \|H^{-1/2}\|_{2\rightarrow 1} \|H^{-1/2}Q^{1/2}\|_{2\rightarrow 2} \nonumber\\
	& \le O\left(\frac{\barkappa}{(1-s)n}\right)\mper\nonumber
	\end{align}
	which in turn implies (by Lemma~\ref{lem:operator_norm}) that $\|H^{-1}Q^{1/2}\|_{\infty\rightarrow \infty} \le O(\sqrt{\frac{\barkappa^2}{(1-s)n}})$. Therefore, %Then we have that $\|H^{-1}\|_{\infty\rightarrow \infty} \le O(\frac{\barkappa^2}{(1-s)n})$. Then we have, 
	%Therefore we obtain that 
	\begin{align}
	\|\xstar -x\|_{\infty} & = \|(s-1)(\xstar-x_{\mle})\|_{\infty} \nonumber\\
	& = (1-s)\|H^{-1}Q^{1/2}Q^{-1/2}\nabla f(\xstar)\|_{\infty} \nonumber \\
	& \le O(\frac{\barkappa}{n}) \cdot \|Q^{-1/2}\nabla f(\xstar)\|_{\infty} \nonumber
	\end{align}

	%\Tnote{need to deal with projectiong to $\one_r$ more carefully. }
	% We can make this very formal but let's really not worry about this right now.
	
	Then by equation~\eqref{eqn:gradient_inftynorm} of Lemma~\ref{thm:gradient_concentration}, we have that 
	\begin{equation}
	\|\xstar -x\|_{\infty} \le O(\frac{\barkappa}{n}) \cdot \|Q^{-1/2}\nabla f(\xstar)\|_{\infty} \le O(\frac{\barkappa}{\sqrt{n}})\mper \nonumber
	\end{equation}
	This contradicts with the fact that $\|\xstar -x\|_{\infty} = \ltau/r$ for some $n \ge \Omega(\barkappa^2 r^2/\ltau^2)$.  Hence we showed that $x_{\mle}$ indeed satisfies that $x_{\mle} \le 2\xstar$, which completes the proof of the claim. 
\end{proof}
\begin{proof}[Proof of Theorem~\ref{thm:rate}]
We restrict out attention to the convex region $\mathcal C = \{x: x\le 2\xstar\}$ which contains both $x_{\mle}$ and $\xstar$. We change the basis of the space according to the Fisher information matrix. Let $g(z) = f(Q^{-1/2}z)$, and let $z^* = Q^{1/2}\xstar$, and $z_{\mle} = Q^{1/2}x_{\mle}$.  Then $z_{\mle}$ is also the maximizer of $g(z)$ and therefore $\nabla g(z_{\mle}) = 0$. Moreover, we have that $\nabla g(z) = Q^{-1/2} \nabla f(z)$ and $\nabla^2 g(z) = Q^{-1/2}\nabla^2 f(z) Q^{-1/2}$. By Theorem~\ref{thm:lower_bound_M}, for any $z\in Q^{1/2}\mathcal C$, we have that $\nabla^2 g(z)\preceq -Q^{-1/2}\cdot nQ/8 \cdot Q^{-1/2} = -n/8\cdot \Id_r$. Therefore $g(z)$ is $\Omega(n)$-strongly concave in $Q^{1/2}\mathcal C$. Then by strong concavity of $g$, we obtain that 
\begin{equation}
\|\nabla g(z^*) \| = \|\nabla g(z^*) - \nabla g(z_{\mle})\| \ge n/2\cdot \|z^* - z_{\mle}\|\mper\nonumber
\end{equation}
By Theorem~\ref{thm:gradient_concentration}, we have $\|\nabla g(z^*) - n Q^{-1/2} \mathbf{1}_r \|= \|Q^{-1/2}(\nabla f(\xstar) - n\mathbf{1}_r)\|\le \widetilde{O}(\sqrt{nr})$. Here we are projecting out the all 1's direction $\mathbf{1}_r$ in the original space because $x$ is a probability simplex and $\inner{\mathbf{1}_r,x}$ is fixed by constraint (similarly the constraint will fix $\inner{Q^{-1/2} \mathbf{1}_r, z}$). Let $\mbox{Proj}$ be the projection to the orthogonal subspace of $Q^{-1/2}\mathbf{1}_r$, we know
\begin{equation}
\|z^* - z_{\mle}\| \le O(1/n)\cdot \|\mbox{Proj} \nabla g(z^*) \|\le \widetilde{O}\left(\sqrt{\frac{r}{n}}\right) \mper\nonumber
\end{equation}

Converting $z$ to $x$, and observing that by Lemma~\ref{lem:convexity_expectation_restricted} $\|AQ^{-1/2}\|_{2\rightarrow 1}\le 1$ and $\|Q^{-1/2}\|_{2\rightarrow 1}\le \barkappa$, we obtain that 
\begin{align}
\|A\xstar - Ax_{\mle} \|_1 & = \|AQ^{-1/2}z^* - AQ^{-1/2}z_{\mle}\|_1 \nonumber\\
&\le\|z^* - z_{\mle}\| \le \widetilde{O}\left(\sqrt{\frac{r}{n}}\right)\mper\nonumber
\end{align}
Similarly, we have 
$
\|\xstar - x_{\mle} \|_1 \le \|Q^{-1/2}(z^* - z_{\mle})\|_1 \le \barkappa\|z^* - z_{\mle}\| \le \widetilde{O}\left(\barkappa\sqrt{\frac{r}{n}}\right)\mper\nonumber
$
%we have that $tu+x_{\mle}\in \mathcal B$ and therefore $H_t \succeq nQ/4$. It follows that $H =  \int_{0}^1 H_t dt \succeq \int_{s}^1 H_t dt \succeq n(1-s)Q/(4\barkappa^2)$. Therefore we obtain that $\|G\|= \|H^{-1}Q^{1/2}\|\le \frac{4\barkappa^2}{(1-s)n}$. 

%\Tnote{Plan was to show: }

%$\|G\|$ is small and therefore $u = G(Q^{-1/2}\nabla f(\xstar))$ has small infinity norm. The problem is that $G$ and $\nabla f(\xstar)$ are not independent. 

%\Tnote{Failure attempt ends}

		%Then $\widehat{H} = Q^{-1/2}HQ^{-1/2}$. 
%		By Lemma~\ref{thm:lower_bound_M}, we have that 
%		$H \succeq nQ/(2C^2)$. It follows that $H^{1/2}Q^{-1}H^{1/2}$
%	
%	\begin{equation}
%	\nabla g(z^*)  = \nabla g(z^*) - \nabla g(z_{\mle}) = \left(\int_{0}^1 \nabla^2 g(t\diffmle)dt \right) \diffmle= -\widehat{H}\diffmle\nonumber
%	\end{equation}
%	
%	where $\widehat{H} = -\left(\int_{0}^1 \nabla^2 g(t\diffmle)dt \right)$. Let $H = -\left(\int_{0}^1 \nabla^2 f(tQ^{-1/2}\diffmle)dt \right)$. Then $\widehat{H} = Q^{-1/2}HQ^{-1/2}$. By Lemma~\ref{thm:lower_bound_M}, we have that 
%$H \succeq nQ/(2C^2)$. It follows that $H^{1/2}Q^{-1}H^{1/2}$

	%\Tnote{There is subtlety here -- we need to prove that $x_{\mle}\le Cx^*$. At least we need to check this. }
	%\authnote{RG}{Yes indeed. Trivial solution involves setting $n = \Omega(r^3)$}
\end{proof}

%Suppose $\supp(x) =  S$ is of size $r$. 

\subsection{Proof of Lemma~\ref{thm:lower_bound_M} and Lemma~\ref{lem:convexity_expectation_restricted}}
We establish equation~\eqref{eqn:psd_lower_bound} by first showing that with high probability, $-\nabla^2 f(\xstar)\ge nQ/2$. Then using the fact that $x\le Cx^*$, we obtain that $-\nabla^2 f(x) \succeq -\nabla^2 f(\xstar)/C^2\succeq nQ/(2C^2) $. Towards obtaining the lower bound for $-\nabla^2 f(\xstar)$, we first consider its expectation, which is lowerbounded in Lemma~\ref{lem:convexity_expectation_restricted}, whose proof is as follows. %The following lemma essentially bound the expectation of $\nabla^2 f(x)$. 
%\begin{lemma}\label{lem:convexity_expectation}
%	%sFor any $x \in \simplex_k$ with $x \le Cx^*$, we have that 
%	The expected value of $M(x^*)$ can be lowerbouned by 
%%	\begin{equation}
%%		\Exp[M(x^*)] = n \sum_{i\in D}  \inner{a_i, x^*} \cdot\frac{a_{i}a_{i}^{\top}}{\inner{a_{i},x}^2}\succeq \frac{n}{C\barkappa^2} \cdot \Id_k
%%	\end{equation}
%	
%		\begin{equation}
%		\Exp[M(x^*)] = n \sum_{i\in D}  \frac{a_{i}a_{i}^{\top}}{\inner{a_{i},x^*}}\succeq \frac{n}{\barkappa^2} \cdot \Id_k\label{eqn:lem:expectation}
%		\end{equation}
%	where the expectation is over the randomness of the generating process of the document. 
%\end{lemma}

%\begin{lemma}\label{lem:convexity_expectation_restricted}
%	%sFor any $x \in \simplex_k$ with $x \le Cx^*$, we have that 
%	%The expected value of $\nabla^2 f(x^*)$ can be lowerbouned by 
%	Fisher information matrix $Q$ satisfies that 
%	\begin{equation}
%	%\|Q^{-1/2}\|_{2\rightarrow 2}\le
%	 \|Q^{-1/2}\|_{2\rightarrow 1} \le \barkappa\mper
%	\end{equation}
%	
%	%	\begin{equation}
%	%		\Exp[M(x^*)] = n \sum_{i\in D}  \inner{a_i, x^*} \cdot\frac{a_{i}a_{i}^{\top}}{\inner{a_{i},x}^2}\succeq \frac{n}{C\barkappa^2} \cdot \Id_k
%	%	\end{equation}
%	As a corollary, 
%	\begin{equation}
%	%\Exp[\nabla^2 f(\xstar)] = 
%	Q \succeq\frac{1}{\barkappa^2} \cdot \Id_r\mper\label{eqn:lem:expectation}
%	\end{equation}
%	%ewhere the expectation is over the randomness of the generating process of the document. 
%\end{lemma}

\begin{proof}[Proof of Lemma~\ref{lem:convexity_expectation_restricted}]
	%\Tnote{Probably need to change notation or have a more general lemma for all coorindates. proof below are for all coordinates}
	
	%\Tnote{$M(\xstar)\leftrightarrow -\nabla^2 f(\xstar)$}
	
	%\Tnote{$a-\rightarrow \hata$}
	
	%	By definition of the generating process $y\vert x$, we have that 
	%	\begin{equation}
	%	\Exp[M(\xstar)] =  n \sum_{i\in [D]} \Pr[w_1 = i]\frac{a_{i}a_{i}^{\top}}{\inner{a_{i},x^*}^2} = n \sum_{i\in [D]}  \frac{a_{i}a_{i}^{\top}}{\inner{a_{i},x^*}} \nonumber
	%	\end{equation}
	For any unit vector $v\in \R^r$, we consider the quadratic form of $v$ over $Q$, %$\Exp[M(\xstar)]$:
	%	\begin{align}
	%	v^{\top} \Exp[M(\xstar)] v %& = n \sum_{i\in [D]} \inner{a_i, x^*}  \cdot \frac{\inner{a_i,v}^2}{\inner{a_{i},x}^2} \nonumber\\
	%	%& \ge n/C\cdot  \sum_{i\in [D]}\frac{\inner{a_i,v}^2}{\inner{a_{i},x}} \nonumber\\
	%	& =  n \sum_{i\in [D]}\frac{\inner{a_i,v}^2}{\inner{a_{i},\xstar}}  \nonumber\\
	%	& =  n \left(\sum_{i\in [D]}\frac{\inner{a_i,v}^2}{\inner{a_{i},\xstar}} \right)\left(\sum_{i \in [D]} \inner{a_i,\xstar}\right) \nonumber\\
	%	& \ge n\left(\sum_{i\in [D]} \inner{a_i,v}\right)^2 = n \|Av\|_1^2 \nonumber\\
	%	& \ge n\|v\|_1^2/\barkappa^2 \ge n/\barkappa^2 \nonumber
	%	\end{align}
	\begin{align}
	v^{\top} Q v %& = n \sum_{i\in [D]} \inner{\hata_i, x^*}  \cdot \frac{\inner{\hata_i,v}^2}{\inner{a_{i},x}^2} \nonumber\\
	%& \ge n/C\cdot  \sum_{i\in [D]}\frac{\inner{\hata_i,v}^2}{\inner{a_{i},x}} \nonumber\\
	& =  \sum_{i\in [D]}\frac{\inner{\hata_i,v}^2}{\inner{\hata_{i},\xstar}}  \nonumber\\
	& =  \left(\sum_{i\in [D]}\frac{\inner{\hata_i,v}^2}{\inner{\hata_{i},\xstar}} \right)\left(\sum_{i \in [D]} \inner{\hata_i,\xstar}\right) \nonumber\\
	& \ge \left(\sum_{i\in [D]} |\inner{\hata_i,v}|\right)^2 = \|\hat{A}v\|_1^2 \label{eqn:eqn12}\\
	& \ge \|v\|_1^2/\barkappa^2 \label{eqn:eqn11}%\ge n/\barkappa^2 \nonumber
	\end{align}
	%where the second inequality uses the fact that $x^*\ge 1/C\cdot x$ and the non-negativity of the entries of $a_i$, $x$ and $x^*$, and 
	where the second line uses the fact that $\sum_i \inner{\hata_i,\xstar} = 1$, the third line uses Cauchy-Schwartz inequality, and the last line uses that $\hat{A}$ has $\ell_1$-condition number $\barkappa$. (Note that $v$ could take negative entries and therefore $\|\hat{A}v\|_1$ could be smaller than $\|v\|_1$. )
	
	Replacing $v$ in equation~\eqref{eqn:eqn11} by $Q^{-1/2}z$, we have $\|z\| \ge \|Q^{-1/2}z\|_1/\barkappa$ for any $z\in \R^r$. Therefore, $\|Q^{-1/2}\|_{2\rightarrow 1}\le \barkappa$. It follows that $\|Q^{-1/2}\|_{2\rightarrow 2}\le \barkappa$, which is equivalent to $Q\succeq \frac{1}{\barkappa^2}\cdot \Id_r$. 
	
	Finally, using the intermediate result equation~\eqref{eqn:eqn12}, plugging $v = Q^{-1/2}z$, we obtain that 
	\begin{equation}
	\|z\|^2 \ge \|\hat{A}Q^{-1/2}z\|_1^2 \nonumber
	\end{equation}
	which implies that $\|\hat{A}Q^{-1/2}\|_{2\rightarrow 1} \le 1$. 
\end{proof}

The following Lemma gives high probability (lower) bound for the spectrum of $\Exp[-\hessian f(\xstar)]$, from which Theorem~\ref{thm:lower_bound_M} follows straightforwardly.  %establish lowerbounds the minimum eigenvalue of $M(x^*)$ with high probability. 
%\begin{proof}[Proof of Lemma~\ref{lem:convexity_expectation}]
%	By definition of the generating process $y\vert x$, we have that 
%	\begin{equation}
%	\Exp[M] =  n \sum_{i\in [D]} \Pr[w_1 = i]\frac{a_{i}a_{i}^{\top}}{\inner{a_{i},x}^2} = n \sum_{i\in [D]} \inner{\hata_i, x^*}  \cdot \frac{a_{i}a_{i}^{\top}}{\inner{a_{i},x}^2} \nonumber
%	\end{equation}
%	Taking an arbitrary unit vector $v\in \R^k$, we consider the quadratic form of $v$ over $\Exp[M]$:
%	\begin{align}
%		v^{\top} \Exp[M] v& = n \sum_{i\in [D]} \inner{\hata_i, x^*}  \cdot \frac{\inner{\hata_i,v}^2}{\inner{a_{i},x}^2} \nonumber\\
%		%& \ge n/C\cdot  \sum_{i\in [D]}\frac{\inner{\hata_i,v}^2}{\inner{a_{i},x}} \nonumber\\
%				& \ge  n/C\cdot  \left(\sum_{i\in [D]}\frac{\inner{\hata_i,v}^2}{\inner{a_{i},x}} \right)\left(\sum_i \inner{\hata_i,x}\right) \nonumber\\
%				& \ge n/C \left(\sum_{i\in [D]} \inner{\hata_i,v}\right)^2 = n/C \cdot \|Av\|_1^2 \nonumber\\
%				& \ge n/C \cdot \|v\|_1^2/\barkappa^2 \ge n/(C\barkappa^2). \nonumber
%	\end{align}
%	where the second inequality uses the fact that $x^*\ge 1/C\cdot x$ and the non-negativity of the entries of $\hata_i$, $x$ and $x^*$, and the fact that $\sum_i \inner{\hata_i,x} = 1$. 
%	The third line uses Cauchy-Schwartz inequality and the last line uses that $A$ has $\ell_1$-condition number $\barkappa$. 
%\end{proof}
\newcommand{\hatQ}{\widehat{Q}}
\begin{lemma}\label{lem:concentration}
	%Suppose $x \in \simplex_k$ satisfies that $x \le Cx^*$ and $x_i\ge \beta/r$ when $x_i > 0$. 
	Let $\hatQ = -\frac{1}{n}\nabla^2 f(\xstar)$ be the empirical fisher information matrix. 
	Suppose $n  =  c_0\barkappa^2r^2/\beta^2\cdot \log k$ for sufficiently large constant $c_0$. Then with probability at least $1-k^{-10}$, we have that 
	\begin{equation}
	%M_R(\xstar)
	\hatQ \succeq\frac{1}{2} \cdot Q \mper\nonumber
	\end{equation}
	Moreover, it holds that 
	\begin{equation}
	\|\hatQ^{-1/2}\|_{2\rightarrow 1} \le \sqrt{2}\barkappa\,, \quad \textup{and} \quad \hatQ \succeq\frac{1}{2\barkappa^2} \cdot \Id_r \nonumber
	\end{equation}
	%and as a consequence
	%	\begin{equation}
	%	%M_R(\xstar)
	%	\hatQ \succeq\frac{1}{2} \cdot Q \mper\nonumber
	%	\end{equation}
	
	%\Tnote{$M_R(\xstar)\leftrightarrow \nabla^2 f(\xstar)$}
\end{lemma}
%\Tnote{Some intuition is needed}

%\Tnote{need to revise below}
Since we have shown that $Q \succeq  1/\barkappa^2\cdot \Id_r$, it would suffice to prove (by matrix concentration inequality) that with high probability, $\hatQ$ concentrates around its expectation with variance significantly less than $n^2$. However, it turns out that the variance $\hatQ$ could have spectral norm as large as $O(r^4n)$. Then it will require $n \gg r^4$ to guarantee that the spectral norm of the variance can be smaller $n^2$, which turns out to be a suboptimal bound. 

We weaken the requirement to $n \gg r^2$ by observing the following inefficiency in the argument above: Though the variance matrix $\Exp[\hatQ^2]$ could have a large eigenvalue with some eigenvector $v$, when this happens the expectation $Q = \Exp[\hatQ]$ also have eigenvalue much larger than $n/\barkappa^2$ around some direction correlated with $v$ (in the meantime it might also have a small eigenvalue $n/\barkappa^2$ in some other direction so that the lower bound $Q \succeq  n/\barkappa^2\cdot \Id_r$ cannot be improved). This suggests us to first whiten the expectation matrix $Q$ to identity matrix, so that we can avoid the inefficiency caused by using spectral norm to measure the size of a very skew matrix. 
\begin{proof}[Proof of Lemma~\ref{lem:concentration}]

	%Let $\hata_w\in \R^r$ be the restriction of $a_w$ to the support $R$. 
	For $w\in y$, let $Z_w  = \frac{\hata_{w}\hata_{w}^{\top}}{\inner{\hata_{w},\xstar}^2}$ be the single term in the sum of $\nabla^2 f(\xstar)$ (see equation~\eqref{eqn:hessian} for the representation of Hessian).  That is, we have $$\hatQ = -\frac{1}{n}\cdot \nabla^2 f(\xstar) = \frac{1}{n}\sum_{w\in y} Z_w
	$$
	%Let $Q = \sum_{i\in [D]}  \frac{\hata_{i}\hata_{i}^{\top}}{\inner{\hata_{i},x^*}}$. By Lemma~\ref{lem:convexity_expectation}, we have that $M_R(\xstar) = \sum_{w\in y} Z_w$ and $\Exp[M_R(\xstar)] = nQ$. Moreover, since $Q = \Exp[M_R(\xstar)]$ is the restriction of $\Exp[M(\xstar)]$ to the indices $R\times R$, by equation~\eqref{eqn:lem:expectation}, we have that $Q \succeq 1/\barkappa^2\cdot \Id_r$. 
	
	%For $w\in y$, let $Z_w  = \frac{\hata_{w}\hata_{w}^{\top}}{\inner{\hata_{w},\xstar}^2}$. 
	
	%For simplicity, for any $j\in [n]$, let $Z_j = \frac{a_{w_j}a_{w_j}^{\top}}{\inner{a_{w_j},x}^2}$. Let $Q = \Exp[Z_1]= \frac{a_{i}a_{i}^{\top}}{\inner{a_{i},x}^2}\succeq \frac{n}{C\barkappa^2}$. 
	
	We change basis using $Q^{-1/2}$. Let $Z'_w = Q^{-1/2}Z_wQ^{-1/2}$. Consider the random variable $\hatQ' = Q^{-1/2}\hatQ Q^{-1/2}\in \R^{r\times r}$, which is a sum of independent random matrices, 
	\begin{equation}
	\hatQ' = \frac{1}{n}\sum_{w\in y}Z_w' \mper\nonumber
	\end{equation}
	
	Toward bounding the fluctuation of $\hatQ$, we apply Bernstein inequality on $\hatQ'$. 
	% %First, since $Z_j \preceq \frac{\|a_{w_j}\|_2^2}{\inner{a_{w_j},x}^2}$ 
	By Lemma~\ref{lem: tech_lem_1}, we know have that $Z_w$ is almost surely bounded by  $\left\|Z_w\right\| \le r^2/\beta^2$. Then using equation~\eqref{eqn:lem:expectation}, 
	\begin{equation}
	Z_w' = Q^{-1/2}Z_wQ^{-1/2} \preceq Q^{-1/2} \cdot r^2/\beta^2\Id_r \cdot Q^{-1/2} = r^2/\beta^2 \cdot Q^{-1}\nonumber
	\end{equation}
	
	Therefore it follows that $Z_w'$ is almost surely bounded by $\left\|Z_w'\right\| \le r^2/\beta^2 \|Q^{-1}\| \le r^2 \barkappa^2/\beta^2$.  Next we  bound the variance of $\hatQ'$: %First we have that 
	\begin{align}
	\Exp[(Z_w')^2] & =  \frac{1}{n}\sum_{i\in [D]} \inner{\hata_i, x^*}  \cdot \left(Q^{-1/2}\frac{\hata_{i}\hata_{i}^{\top}}{\inner{\hata_{i},\xstar}^2}Q^{-1/2} \right)^2\nonumber \\
	& =  \frac{1}{n}\sum_{i\in [D]} \frac{\hata_i^{\top}Q^{-1}\hata_i}{\inner{\hata_i,\xstar}^2}\cdot  Q^{-1/2}\frac{\hata_{i}\hata_{i}^{\top}}{\inner{\hata_{i},\xstar}}Q^{-1/2} \nonumber \\
	& \preceq \frac{1}{n}\sum_{i\in [D]} \frac{\|\hata_i\|_2^2 \|Q^{-1}\|}{\inner{\hata_i,\xstar}^2}\cdot  Q^{-1/2}\frac{\hata_{i}\hata_{i}^{\top}}{\inner{\hata_{i},\xstar}}Q^{-1/2} \nonumber \\
	%& \preceq  \sum_{i\in [D]} \inner{a_i, x^*} \cdot \frac{C\barkappa^2 r^2/\beta^2}{\inner{a_i,x}^2}\cdot Q^{-1/2} \frac{a_{i}a_{i}^{\top}}{\inner{a_{i},x}^2}Q^{-1/2} \nonumber \\
	& \preceq \frac{1}{n}\cdot 1/\barkappa^2 \cdot r^2/\beta^2 \cdot \sum_{i\in [D]}Q^{-1/2}\frac{\hata_{i}\hata_{i}^{\top}}{\inner{\hata_{i},\xstar}}Q^{-1/2}  \nonumber\\
	& = \barkappa^2 \cdot r^2/\beta^2 \cdot  Q^{-1/2}QQ^{-1/2}  = \barkappa^2r^2/\beta^2 \cdot \Id_r\nonumber
	%& = \barkappa^2 r^2/\beta^2Q^{-1/2}QQ^{-1/2} = C\barkappa^2 r^2/\beta^2\Id_r
	\cdot 
	%\frac{\|a_i\|_2^2 a_ia_i^{\top}}{\inner{a_{i},x}^4}
	\end{align} 
	
	where the second line is just a re-arrangement of the first line, %uses the fact that $\supp(\xstar)  =R$ and therefore $\inner{a_i,\xstar} = \inner{\hata_i,\xstar}$, 
	and the  third line uses the definition of spectral norm, and the fourth line uses that $\frac{\|\hata_i\|_2^2}{\inner{\hata_i,\xstar}^2}\le r^2/\beta^2$ (see Lemma~\ref{lem: tech_lem_1}) and $Q\succeq 1/\barkappa^2 \cdot \Id_r$, and finally the fifth line uses the definition of $Q$. 
	
	%\Tnote{We lost something at the third line since dual to the convertion of the norm. Doesn't seem to be improvable}
	
	%Therefore we have that $\left\|\Exp[\sum_j (Z_j')^2] \right\|\le N$. 
	
	Now we are ready to apply Bernstein inequality on $\hatQ'$. We conclude that with high probability,  %we have that with high probability, it holds that 
	
	\begin{equation}
	\left\|\hatQ' - \Exp[\hatQ']\right\| \le O(\sqrt{\barkappa^2 r^2/\beta^2\cdot n\log k} + r^2\barkappa^2/\beta^2\cdot \log k)\nonumber
	\end{equation}

	Recall the definition of $\hatQ' = Q^{-1/2} \hatQ Q^{-1/2}$ and that $\Exp[\hatQ'] = Q^{-1/2}\Exp[Q]Q^{-1/2} = \Id_r$ , we obtain that for $n\ge c_0r^2\barkappa^2/\beta^2\cdot \log k$ with sufficiently large constant $c_0$, 
	\begin{equation}
	\frac{1}{2}\cdot \Id_r \preceq \hatQ'\preceq \frac{3}{2}\cdot \Id_r\label{eqn:eqn13}
	\end{equation}
	%\begin{equation}
	% - n/2\cdot \Id_r \preceq Q^{-1/2} \left(\hatQ- \Exp[\hatQ]\right) Q^{-1/2} \preceq n/2\cdot \Id_r\nonumber
	%\end{equation}
	
	Therefore towards upperbounding the operator norm $\|\hatQ^{-1/2}\|_{2\rightarrow 1}$, we consider the quadratic form $v^{\top}\hatQ v$ and have that 
	\begin{align}
	v^{\top}\hatQ v & = (Q^{1/2}v)^{\top} \cdot \hatQ' \cdot (Q^{1/2}v) \ge \frac{1}{2}\|Q^{1/2}v\|^2 \ge \frac{1}{2\barkappa^2}\|v\|_1^2  \nonumber
	\end{align}
	where the first inequality uses~\eqref{eqn:eqn13} and the second uses Lemma~\ref{lem:convexity_expectation_restricted} (or equation~\eqref{eqn:eqn11}). It follows easily that $\|\hatQ^{-1/2}\|_{2\rightarrow 1} \le \sqrt{2}\barkappa$. 
	
	%It follows that $M_R(\xstar) -\Exp[M_R(\xstar)] \succeq -n/2 \cdot Q$. Recall that $\Exp_R[M(\xstar)] = nQ$. Therefore, equation above implies that
	
	%It follows $\Exp[M] = Q$ that 
	
	%\begin{equation}
	%M_R(\xstar) \succeq n/2\cdot Q \label{eqn:eqn10}%\succeq %(n - O(\sqrt{C\barkappa^2 r^2/\beta^2n\log k} + r^2C\barkappa^2/\beta^2\log k)) Q\nonumber
	%\end{equation}
	%which in turn implies that \begin{equation}
	%M_R(\xstar) \succeq n/2\cdot Q \succeq n/(2\barkappa^2) \cdot \Id_r\nonumber
	%\end{equation}
	%for some $n = O(C\barkappa^2r^2\log k/\beta^2)$. 
\end{proof}
\begin{lemma}\label{lem: tech_lem_1}
	Suppose  $x\in \simplex_r$ such that $x\ge \beta/r \cdot \one_r$, then 
	%$\forall i, x_i\ge \beta/r$ if $x_i > 0$. 
	Then for any vector $a\in \R_{\ge 0}^r$, $\|a\|_2 \le \beta/r\cdot \inner{a,x}$ and $\left\|\frac{aa^{\top}}{\inner{a,x}^2}\right\| \le r^2/\beta^2$. 
\end{lemma}

\begin{proof}[Proof of Lemma~\ref{lem: tech_lem_1}]
	%Let $R$ be the support of $x$. 
	Then we have that $\inner{a,x} \ge \beta/r \cdot \|a\|_1 \ge \beta/r\cdot \|a\|_2$, and therefore $\left\|\frac{aa^{\top}}{\inner{a,x}^2}\right\| \le \frac{\|a\|_2^2}{\inner{a,x}^2}\le r^2/\beta^2$. 
\end{proof}

\begin{lemma}\label{lem:operator_norm}
	For any symmetric positive semidefinite matrix $G$, it holds that $\|G\|_{\infty \rightarrow \infty}\le \|G\|_{2\rightarrow 1}$. 
\end{lemma}

\begin{proof}[Proof of Lemma~\ref{lem:operator_norm}]
	We have that since $\ell_1$ norm is larger than $\ell_2$, we have that $\|G\|_{1\rightarrow 1}\le \|G\|_{2\rightarrow2}$. For symmetric matrix, we have $\|G\|_{\infty\rightarrow \infty } = \|G\|_{1\rightarrow 1}$ which completes the proof. 
\end{proof}

\subsection{Proof of Lemma~\ref{thm:gradient_concentration}}

%When there are enough samples, we showed the Hessian matrix near $x^*$ is negative definite. In this section we combine this fact with concentration of the gradient to show the Maximum Likelihood solution conditioned on the correct support $R$ is close to $x^*$.

%Recall $Q = \E[M(x^*)]$, and we showed for any point $x$ that has support $R$ and entries at least $(1\pm \epsilon)x^*_i$, $(1-2\epsilon) Id \preceq Q^{-1/2} M(x) Q^{-1/2} \preceq (1+2\epsilon) Id$.

%L %et $u$ be the gradient of $f(x^*)$, restricted to the support $R$. We shall show the norm of $Q^{-1/2}u$ is concentrated. This will allow us to bound the distance of $x^*$ to the optimal point of $f$.

Before proving the concentration of $\nabla f(\xstar)$, we start by calculating its mean and variance. %$of the $\nabla f(\xstar)$. 

\begin{lemma}\label{lem:mean_variance}
	The mean and variance of $\nabla f(\xstar)$ are equal to 
	\begin{equation}
	\Exp\left[\nabla f(\xstar)\right] = n\cdot \one_r\,,\nonumber
	\end{equation}
	and 
	%$\nabla f(x) = $
%The expectation of $u$ $\E[u]$ is a multiple of all $1$'s vector $\vec{1}$, 
\begin{equation}
\Exp\left[\|\nabla f(\xstar)\|_{Q^{-1}}^2\right] = nr\,.\nonumber
\end{equation}

\end{lemma}

\begin{proof}
Plugging $x = \xstar$ into equation~\eqref{eqn:hessian}, we have that 
$$
\nabla f(\xstar) = \sum_{w\in y} \frac{\hat{a}_w}{\inner{\hat{a}_w,x^*}}\,.
$$
It follows straightforwardly that 

\begin{align}
\Exp\left[\nabla f(\xstar)\right] &= \sum_{w\in y} \Exp\left[\frac{\hat{a}_w}{\inner{\hat{a}_w,x^*}}\right] \nonumber\\
& = n\sum_{i\in [D]}\inner{\hat{a}_i,x^*} \frac{\hat{a}_i}{\inner{\hat{a}_i,x^*}}  = n\sum_{i\in [D]} \hat{a}_i = n\one_r\mper\nonumber
\end{align}

%We only need to compute the expectation of one word, which is equal to
%$$
%\sum_{i\in [D]} \inner{\hat{a}_i,x^*} \frac{\hat{a}_w}{\inner{\hat{a}_w,x^*}} = \sum_{i\in [D]} \hat{a}_w = \vec{1}.
%$$

Towards computing the variance of $\nabla f(\xstar)$ (under the $Q^{-1}$ norm), we observe that 
\begin{align}
\Exp\left[\nabla f(\xstar)\nabla f(\xstar)^{\top}\right] & = \sum_{i\in [D]}\inner{\hat{a}_i,x^*} \frac{\hat{a}_i\hat{a}_i^{\top}}{\inner{\hat{a}_i,x^*}^2}\nonumber\\
& = -\Exp\left[\nabla^2 f(\xstar)\right]] = Q\mper\nonumber
\end{align}
Therefore %$\E[Q^{-1/2}uu^TQ^{-1/2}] = nId_r$, and 
$$\E[\|\nabla f(\xstar)\|_{Q^{-1}}^2] = \Exp[\trace(Q^{-1}\nabla f(\xstar)\nabla f(\xstar)^{\top})]= nr\mper$$
\end{proof}

%This together with a bound on $\|Q^{-1/2} \frac{\hat{a}_w}{\inner{\hat{a}_w,x^*}}\|_2$ will show

Now we are ready to prove Lemma~\ref{thm:gradient_concentration} by using Bernstein inequality on $\nabla f(\xstar)$. 
\begin{proof}[Proof of Lemma~\ref{thm:gradient_concentration}]
	Note that $\|\nabla f(\xstar)\|_{Q^{-1}}^2 = \|Q^{-1/2} \nabla f(\xstar)\|^2$. We apply concentration inequality on $$Q^{-1/2} \nabla f(\xstar) = \sum_{w\in y} \frac{Q^{-1/2}\hat{a}_w}{\inner{\hat{a}_w,x^*}}\,, $$
	which is a sum of independent random variables. Toward using Bernstein inequality, we first verify that \begin{equation}
	\left\| \frac{Q^{-1/2}\hat{a}_w}{\inner{\hat{a}_w,x^*}}\right\| \le \sigma^{-1/2}_{\min}(Q) \cdot  \left\|\frac{\hat{a}_w}{\inner{\hat{a}_w,x^*}}\right\|  \le \barkappa \cdot r/\ltau \nonumber
	\end{equation}
	where the last inequality uses the 
	As bounded in Lemma~\ref{lem:mean_variance}, the variance of $Q^{-1/2}\nabla f(\xstar)$ is at most $nr$, and the mean of  $Q^{-1/2}\nabla f(\xstar)$ is $nQ^{-1/2}\one_r$. Therefore, by Bernstein inequality, we obtain that with high probability
	\begin{equation}
		\left\|Q^{-1/2}(\nabla f(\xstar) - n\one_r)\right\| \le \widetilde{O}(\barkappa r/\ltau+ \sqrt{nr}) = \widetilde{O}(\sqrt{nr})\mper\nonumber
	\end{equation}
	for $n \ge \Omega(\barkappa^2 r/\ltau^2)$. 
	%\ % %Tnote{We only need $n\ge r$ here a bit surprisingly}	
	
	For the infinity norm bound, we fix a coordinate $j$ and have that the $j$-th coordinate of $\frac{Q^{-1/2}\hat{a}_w}{\inner{\hat{a}_w,x^*}}$ is smaller than its $\ell_2$ norm, which has been shown to be smaller than $\barkappa r/\ltau$. Therefore applying Bernstein's inequality, and taking union bound over all coordinates, we have that 
	\begin{align}
		\left\|Q^{-1/2}(\nabla f(\xstar) - n\one_r)\right\|_{\infty} & \le O(\barkappa r/\ltau\cdot \log^2 k+ \sqrt{n\log^3 k}) \nonumber\\
		&= O(\sqrt{n}\log^{1.5}k)\mper\nonumber
	\end{align}
	for some $n \ge \Omega((\barkappa^2r^2\log^2 k)/\ltau^2)$
\end{proof}

%On the other hand, we know within a region the function $f(Q^{-1/2}z)$ is concave with negative Hessian between $(1\pm \epsilon)nId_r$. This means the optimal $z$ and $z^* = Q^{1/2}x^*$ is within distance $O(\sqrt{r/n})$.

%Converting this to the original distance, we know $\|x-x^*\| \le \|Q^{-1/2}(z-z^*)\| = O(\barkappa \sqrt{r/n})$.

\section{Proof of Proposition~\ref{prop:condition_number}}\label{sec:duality}
\begin{proof}[Proof of Proposition~\ref{prop:condition_number}]
Let $J$ be the all 1's matrix. 	We rewrite the program~\eqref{eqn:program} as an LP by introducing auxiliary variable $t$.   
	\begin{align}
	\lambda_{\delta}(A) = \textup{min} \quad & t \nonumber\\
	\textup{s.t.} \quad & B \le tJ \nonumber\\
	& -B \le -t J\nonumber\\
	& BA-\Id \le \delta J\nonumber\\
	& -BA+\Id \le - \delta J\nonumber
	\end{align}	 
	Let $P_1, P_2\in \R^{k\times D}, Q_1,Q_2\in \R^{k\times k}$ be the dual variables for the four (set of) constraints. Let $\inner{X,Y} =\textup{tr}(X^TY)$ denote the inner product of the two matrices. Then, the dual of the program above is 
	
	\begin{align}
	\textup{maximize} \quad &  \inner{Q_2-Q_1,\Id} - \delta \inner{Q_2+Q_1,J}\nonumber\\
	\textup{s.t.} \quad & (P_1-P_2)+(Q_1-Q_2)A^{\top} = 0 \nonumber\\
	&  \inner{P_1+P_2, J} = 1\nonumber\\
	& P_1,P_2,Q_1,Q_2 \ge 0  \label{eqn:eqn3}
	\end{align}
	Let $Q  = Q_2 - Q_1$ and $P  = P_1 - P_2$. Observe that $\inner{P_1+P_2, J} \ge \veconenorm{P}$ and  $\inner{Q_1+Q_2, J} \ge \veconenorm{Q}$, it is easy to verify that program~\eqref{eqn:eqn3} is equivalent to the program below 
	\begin{align}
	\textup{maximize} \quad & \textup{tr}(Q) - \delta \veconenorm{Q}\nonumber\\
	\textup{s.t.} \quad & P+QA^{\top} = 0 \nonumber\\
	&  \veconenorm{P} \le 1 \label{eqn:eqn14}
%	& P_1,P_2,Q_1,Q_2 \ge 0 
	\end{align}
	Towards further simplification, we claim that program~\eqref{eqn:eqn14} is equivalent to the following (non-convex) program with vector variables $x\in \R^k$: 
	 \begin{align}
	 \textup{maximize} \quad & \|x\|_{\infty} - \delta \|x\|_1 \nonumber\\ 
	 \textup{s.t.} \quad & \|Ax\|_1 \le 1 \label{eqn:eqn5}
	 %P+QA^{\top} = 0 \\
	 %&  \veconenorm{P} \le 1\\
	 %	& P_1,P_2,Q_1,Q_2 \ge 0 
	 \end{align}
	 \newcommand{\xopt}{x_{\textup{opt}}}
	 Indeed, suppose program~\eqref{eqn:eqn2} has optimal value $\lambda$ and program~\eqref{eqn:eqn5} has optimal value $\lambda'$ with optimal solution $\xopt$. 
	 We first show that for any $x$, \begin{equation}
	 \|x\|_{\infty}-  \delta \|x\|_1\le \lambda' \|Ax\|_1\,,\label{eqn:eqn7}
	 \end{equation}which is due to the homogeneity of the equation.  
	  Then, consider any $P,Q$ that satisfies the constraint of~\eqref{eqn:eqn2}. Let $P_j$ and $Q_j$ be the rows of $P$ and $Q$. We have 
	  \begin{align}
	  \textup{tr}(Q) - \delta \veconenorm{Q} & \le \sum_j \left(\|Q_j\|_{\infty} - \delta\|Q_j\|_1\right) \nonumber\\
	  & \le \sum_j \left(\|AQ_j\|_1\right) = \veconenorm{P}\nonumber
	  \end{align} where the second inequality is by equation~\eqref{eqn:eqn7}. Therefore $\lambda \le \lambda'$. 
	  
	  On the other hand, suppose the $\xopt$ has coordinate $i$ with the largest absolute value. Then let $Q$ be the matrix which has it's $j$-th row as $\xopt$ and 0 elsewhere, and $P = -QA^{\top}$. Then it's straightforward to check $P,Q$ satisfy the constraint of~\eqref{eqn:eqn2} and have objective value $\lambda'$. Therefore $\lambda' \le \lambda$. Hence we obtained that $\lambda = \lambda'$. Finally, from~\eqref{eqn:eqn5} it's easy to see $\lambda_0(A) = \lambda(A)$, and $\lambda_{\delta}(A)\le \lambda_0(A)$. 
	 
	 %Since for any $x$, $ \|x\|_{\infty} - \delta \|x\|_1 \le \|x\|_{\infty}\le \lambda(A) \|Ax\|_1 \le \lambda(A)$, we got the desired result. 
\end{proof}
\section{Missing proofs in Section~\ref{sec:lowerbound}} \label{sec:app:lowerbound}
\begin{proof}[Proof of Lemma~\ref{lem:lower_bound}]
	We construct $x,x^{-}$  randomly as follows. Let $R \subset [k]$ be a random subset of size $r$, and let $R^- \subset R$ be a random subset with one item removed from $R$ (so $|R^-| = r-1$). Let $x_i = 1/r$ if $i\in R$, and $x_i = 0$ otherwise. Let $x^-_i = 1/(r-1)$ if $i\in R^-$ and 0 otherwise. 
	
	Each word of the document is from a multinomial distribution whose probabilities are specified by either $Ax$ or $Ax^-$. We shall show the two distributions have small KL-divergence so no algorithm can distinguish between the two.
	
	First we observe that most rows will have between $r/4$ and $3r/4$ entries with value roughly $2/D$ in the subset $R^-$ of coordinates . We say a row is biased if it has less than $r/4$ or more than $3r/4$ nonzero entries in $R^-$, otherwise it's balanced. 
	
	\begin{claim}
		With high probability when $r$ is larger than a fixed constant, and when $D \gg r^2$, there are at most $D/r^2$ biased rows.
	\end{claim}
	
	\begin{proof}
		For every row, when we look at the entries in $R^-$, these entries are independent and has probability $1/2$ of being nonzero. Therefore the probability that the row is biased is $e^{-\Omega(r^2)}$ which is much smaller than $1/2r^2$. Different rows are also independent, so by Chernoff bound we know with high probability there are at most $D/r^2$ biased rows.
	\end{proof}
	
	We know the entries in $i$-th column of $A$ are all equal to $1/|S_i|$, and with high probability $1/|S_i|$ is within $(2\pm o(1))/D$. Therefore the probability of every word in $Ax$ or $Ax^-$ is at most $(2+o(1))/D$, and the probability that we see a biased row is less than $o(1)$ if we have $o(r^2)$ samples. Therefore we can condition on the event that the algorithm does not see any biased row.
	
	Suppose row $i$ that is balanced, let $\{j\} = R\backslash R^-$, we know
	$
	|(Ax)_i - (Ax^-)_i| = |\frac{1}{r}A_{i,j} -\frac{1}{r(r-1)}\sum_{t\in R^-} A_{i,t}| \le \max A_{i,j}\cdot \frac{2}{r} < \frac{5}{rD}.
	$
	
	On the other hand, we know $(Ax^-)_i\in [(1-o(1))/2D, (3+o(1))/2D]$ because row $i$ is balanced. Let $p_i$ be the probability of seeing word $i$ from $Ax$, conditioned on that we don't see any biased row(word). Similarly let $q_i$ be the probability of seeing word $i$ from $Ax^-$ with the same conditioning. The probabilities $p_i$ and $q_i$ are within $1+1/r^2$ multiplicative factor with $(Ax)_i$ and $(Ax^-)_i$ by the rule of conditioning, and in particular we know $(p_i-q_i) \le 10q_i/r$. Therefore the $\chi^2$-distance between $p$ and $q$ is
	$
	\chi^2(p,q)  = \sum_{i} \frac{(p_i-q_i)^2}{q_i} \le \sum_{i} 100q_i/r^2 = 100/r^2\,,
	$
	where the sum is over all balanced row $i$. 
	This implies that the KL-divergence between $p$ and $q$ is less than $100/r^2$ and therefore it is impossible to distinguish between $p$ and $q$ with more than $1/2+o(1)$ accuracy with $o(r^2)$ samples.
	
\end{proof}

Next we prove Lemma~\ref{lem:difference}.

\begin{proof}[Proof of Lemma~\ref{lem:difference}]
For the lowerbound of $\kappa(A)$, we will construct a vector $x = (1,1,1,...,1,-1,...,-1)$. That is, $x_i = 1$ for $i \le k/2$ and $x_i = -1$ if $i > k/2$ (without loss of generality here we assume $k$ is even).

Now consider the $\ell_1$ norm of $Ax$. To do that, consider a different matrix $\hat{A}$ where $\hat{A}_{i,j} = 2/D$ if $A_{i,j} > 0$, and $\hat{A}_{i,j} = 0$ if $A_{i,j} = 0$. Basically, $\hat{A}$ has the same support as $A$, except its row may not normalized.

Using Chernoff bound and Union bound, we know with high probability, the size of the sets $S_i$'s are bounded by $D/2 \pm O(\sqrt{D\log k})$. Therefore the maximum entry in $\hat{A} - A$ is at most $\frac{2}{D} \cdot \sqrt{(\log k)/D}$. Therefore $|(\hat{A} - A)x|_1 \le D \cdot \frac{2}{D} \cdot \sqrt{(\log k)/D} \cdot |x|_1 \le O(\sqrt{(\log k)/D}) |x|_1$.

On the other hand, we know for the $x$ that we constructed, $\E[\hat{A}x] = 0$, and $\E[|(\hat{A}x)_i|] = O(\sqrt{k}/D)$ because entries in $\hat{A}$ are independent of each other. By Chernoff bounds we know with high probability $|\hat{A}x|_1 = O(\sqrt{k})$. However, $|x|_1 = k$, so when $D \gg k\log k$ we have

$$
|Ax|_1 \le |\hat{A}x|_1 + |(\hat{A} - A)x|_1 \le O(\sqrt{k}) = O(\frac{1}{\sqrt{k}}) |x|_1.
$$
Therefor $\kappa(A) \ge \Omega(\sqrt{k})$. 

For $\lambda(A)$, first notice that by Lemma~\ref{lem:inftyonecondition}, we know $\Lambda_{\delta}(A) \le 1$ when $D$ is larger than $\Omega((\log k)/\delta^2$.
That means there is a matrix $B$ with maximum entry at most 1 and $BA = I - \Delta$ where $\Delta$ has maximum entry $\delta$. Now let $\delta < 1/2k$, then the rows of $\delta$ have $\ell_1$ norm at most $1/2$. By Gershgorin's Disk Theorem we know $\|\Delta\|< 1/2$. Therefore we can write

$$
(I-\Delta)^{-1} = I + \sum_{i=1}^\infty \Delta^i =: I + C.
$$

The matrix $C$ is defined to be $\sum_{i=1}^\infty \Delta^i$. The maximum entry $|\Delta^i|_\infty$ is bounded by $(1/2)^{i-1} \delta$. Therefore the maximum entry $|C|_\infty$ is bounded by $2\delta$. Now let $\hat{B} = (I+C) B$, we know $\hat{B}A = (I+C)BA = (I-\Delta)^{-1} (I-\Delta) = I$. On the other hand,

$$
|\hat{B}|_\infty \le |B|_\infty + |CB|_\infty \le 1 + |C|_\infty |B|_\infty\cdot k \le 1+2\delta k \le 2.
$$

So $A$ has a pseudoinverse with maximum entry at most 2. By Proposition~\ref{prop:condition_number} the condition number $\lambda(A) \le 2$.
\end{proof}

\section{Lower bound on $\ell_1\to\ell_1$ condition number}\label{sec:condition_number}

Here we describe how we bound from below  the $\ell_1\to \ell_1$ condition number in Table~\ref{tab:condition_number}. 
Fix $\delta > 0$ and let $B_{\delta} \in \mathbb{R}^{k \times D}$ be a minimizer of LP \eqref{eqn:program} with given $\delta$, i.e. $|B_{\delta}|_{\infty} = \lambda_{\delta}$ and $|BA - I|_{\infty} \le \delta$. By compactness, there exists a $v$ such that $|v|_{\infty}/|Av|_1 = \lambda(A)$. Then,
\begin{align*}
\lambda_{\delta}(A)
&\ge \frac{|B_{\delta} A v|_{\infty}}{|A v|_1} \\
&\ge (|v|_{\infty} - |(BA - I)v|_{\infty})/|Av|_1 \\
&\ge (|v|_{\infty} - |BA - I|_{\infty} |v|_1)/|Av|_1 \\
&\ge \lambda(A) - \delta \kappa(A).
\end{align*}
Hence, we can bound from below $\kappa(A)$ by, 
\[ \kappa(A) \ge \frac{\lambda_{\delta}(A) - \lambda(A)}{\delta}. \]

\end{document}